\documentclass[letterpaper, 10 pt, conference]{ieeeconf}  % Comment this line out
\IEEEoverridecommandlockouts                              % This command is only
\makeatletter
\let\NAT@parse\undefined
\makeatother

\font\stixfrak=stix-mathfrak at 10pt
-mathcal at 10pt

\renewcommand{\baselinestretch}{0.91}

\usepackage[dvipsnames]{xcolor}

\newcommand*\linkcolours{ForestGreen}
\usepackage{optidef} %added by ulises 063020
\usepackage{times}
\usepackage{graphicx}
\usepackage{amsmath}
\usepackage{amsfonts}
\usepackage{amssymb}
\usepackage{amsxtra}
\usepackage{amsthm} % Unnecessary in LNCS
\usepackage{mathrsfs}
\usepackage{stmaryrd}

\usepackage{breakurl}

\usepackage[colorlinks=true,urlcolor=black,linkcolor=blue,citecolor=blue]{hyperref}
\usepackage[capitalize,nameinlink]{cleveref}
\hypersetup{
colorlinks,
linkcolor=\linkcolours,
citecolor=\linkcolours,
filecolor=\linkcolours,
urlcolor=\linkcolours}

\usepackage[linesnumbered,lined,boxed,commentsnumbered]{algorithm2e}

\SetDataSty{textnormal}
\SetCommentSty{textnormal}
\usepackage{algpseudocode} 
\usepackage{listings}
\let\labelindent\relax
\usepackage{enumitem}
\usepackage[labelfont={bf},font=small]{caption}
\usepackage[none]{hyphenat}

\usepackage{mathtools, cuted}

\usepackage{tabularx}

\usepackage{float}

\usepackage{pifont}% http://ctan.org/pkg/pifont

\newcolumntype{Y}{>{\centering\arraybackslash}X}

\usepackage[]{placeins}

\newcommand\extraspace{3pt}

\usepackage{placeins}

\usepackage{tikz}

\usepackage[framemethod=tikz]{mdframed}

\usepackage{afterpage}

\usepackage{stfloats}

\usepackage{comment}

\newcounter{MYtempeqncnt}

\makeatletter
\newcommand{\removelatexerror}{\let\@latex@error\@gobble}
\makeatother

\makeatletter
\def\@endtheorem{\endtrivlist}
\makeatother

\newtheorem{definition}{Definition}
\newtheorem{proposition}{Proposition}

\newtheorem{corollary}{Corollary}

\newtheorem{lemma}{Lemma}
\newtheorem{problem}{Problem}

\newcommand{\nn}{{\mathscr{N}\negthickspace\negthickspace\negthinspace\mathscr{N}}\negthinspace}
\newcommand{\ou}{%
  \mathrel{%
    \vcenter{\offinterlineskip
      \ialign{##\cr$<$\cr\noalign{\kern-1.5pt}$>$\cr}%
    }%
  }%
}%

\newcommand{\overbar}[1]{\mkern 3.0mu\overline{\mkern-2.5mu#1\mkern-2.0mu}\mkern 2.0mu}

\newcommand{\lblkbrbrack}{\negthinspace\text{{\stixfrak\char"36}}\normalfont}
\newcommand{\rblkbrbrack}{\text{{\stixfrak\char"37}}\normalfont}

\newcommand{\subarg}[1]{\lblkbrbrack #1 \rblkbrbrack}

\newcommand{\optprob}[1]{{\bfseries\texttt{#1}}}

\begin{document}

\title{\LARGE \bf
%Repairing Neural Networks for Nonlinear Feedback Control Systems
Safe-by-Repair: A Convex Optimization Approach for Repairing Unsafe Two-Level Lattice Neural Network Controllers
}

%\author{ \parbox{3 in}{\centering Huibert Kwakernaak*
%         \thanks{*Use the $\backslash$thanks command to put information here}\\
%         Faculty of Electrical Engineering, Mathematics and Computer Science\\
%         University of Twente\\
%         7500 AE Enschede, The Netherlands\\
%         {\tt\small h.kwakernaak@autsubmit.com}}
%         \hspace*{ 0.5 in}
%         \parbox{3 in}{ \centering Pradeep Misra**
%         \thanks{**The footnote marks may be inserted manually}\\
%        Department of Electrical Engineering \\
%         Wright State University\\
%         Dayton, OH 45435, USA\\
%         {\tt\small pmisra@cs.wright.edu}}
%}

\author{Ulices Santa Cruz$^{1}$, James Ferlez$^{1}$, and Yasser Shoukry$^{1}$% <-this % stops a space
\thanks{
This work was partially sponsored by the NSF awards \#CNS-2002405 and \#CNS-2013824.
}
\thanks{$^{1}$Ulices Santa Cruz, James Ferlez, and  Yasser Shoukry are with the Department of Electrical Engineering and Computer Science, University of California Irvine, Email: \{usantacr, jferlez,yshoukry@uci.edu}%
}

% \author{Jeffrey M. Ede$^{1}$ and Richard Beanland$^{2}$% <-this % stops a space
% \thanks{*This work was not supported by any organization}% <-this % stops a space
% \thanks{$^{1}$H. Kwakernaak is with Faculty of Electrical Engineering, Mathematics and Computer Science,
%         University of Twente, 7500 AE Enschede, The Netherlands
%         {\tt\small h.kwakernaak at papercept.net}}%
% \thanks{$^{2}$P. Misra is with the Department of Electrical Engineering, Wright State University,
%         Dayton, OH 45435, USA
%         {\tt\small p.misra at ieee.org}}%
% }

\maketitle

%%%%%%%%%%%%%%%%%%%%%%%%%%%%%%%%%%%%%%%%%%%%%%%%%%%%%%%%%%%%%%%%%%%%%%%%%%%%%%%%
\begin{abstract}

% In this paper, we consider the problem of repairing a Rectified Linear Unit (ReLU) Neural Network (NN) with Two Level Lattice (TLL) Architecture used as state-feedback NN controller for discrete-time input-affine nonlinear systems.
% Given a counter example and Safety specifications our general approach leverages current non-linear reachability solvers to upperbound the admissible change on the neural network, then by exploiting the nature of RELU TLL Architecture we use convex optimization to find the minimal weight changes. An experiment is provided where we repair the NN controller of a closed-loop car model.

In this paper, we consider the problem of \emph{repairing} a data-trained Rectified Linear Unit (ReLU) Neural Network (NN) controller for a discrete-time, input-affine system. That is we assume that such a NN controller is available, and we seek to repair unsafe closed-loop behavior at one known ``counterexample'' state while simultaneously preserving a notion of safe closed-loop behavior on a separate, verified set of states. To this end, we further assume that the NN controller has a Two-Level Lattice (TLL) architecture, and exhibit an algorithm that can systematically and efficiently repair such an network. Facilitated by this choice, our approach uses the unique semantics of the TLL architecture to divide the repair problem into two significantly decoupled sub-problems, one of which is concerned with repairing the un-safe counterexample -- and hence is essentially of local scope -- and the other of which ensures that the repairs are realized in the output of the network -- and hence is essentially of global scope. We then show that one set of sufficient conditions for solving each these sub-problems can be cast as a convex feasibility problem, and this allows us to formulate the TLL repair problem as two separate, but significantly decoupled, convex optimization problems. Finally, we evaluate our algorithm on a TLL controller on a simple dynamical model of a four-wheel-car.
% The NN repair problem is to alter the parameters of the NN -- i.e. ``repair'' it -- to fix this incorrect behavior, while simultaneously preserving the correct behavior. Given a counter example and Safety specifications our general approach leverages current non-linear reachability solvers to upperbound the admissible change on the neural network, then by exploiting the nature of RELU TLL Architecture we use convex optimization to find the minimal weight changes. An experiment is provided where we repair the NN controller of a closed-loop car model.

\end{abstract}

% \thispagestyle{empty}
% %\pagestyle{empty}
% \pagestyle{plain}

% !TEX root = ./main.tex

%%%%%%%%%%%%%%%%%%%%%%%%%%%%%%%%%%%%%%%%%%%%%%%%%%%%%%%%%%%%%%%%%%%%%%%%%%%%%%%%
\section{INTRODUCTION}

The proliferation of Neural Networks (NNs) as safety-critical controllers has made obtaining provably correct NN controllers vitally important. However, most current techniques for doing so involve a repeatedly training and verifying a NN until adequate safety properties have been achieved. Such methods are not only inherently computationally expensive (because training and verification of NNs are), their convergence properties can be extremely poor. For example, when verifying multiple safety properties, such methods can cycle back and forth between safety properties, with each subsequent retraining achieving one safety property by undoing another one.

An alternative approach obtains safety-critical NN controllers by \emph{repairing} an existing NN controller. Specifically, it is assumed that an already-trained NN controller is available that performs in a \emph{mostly} correct fashion, albeit with some specific, known instances of incorrect behavior. But rather than using retraining techniques, repair entails \emph{systematically} altering the parameters of the original controller \emph{in a limited way}, so as to retain the original safe behavior while simultaneously correcting the unsafe behavior. The objective of repair is to exploit as much as possible the safety that was learned during the training of the original NN parameters, rather than allowing re-training to \emph{unlearn} safe behavior.

Despite these advantages, the NN repair problem is challenging because it has two main objectives, both of which are at odds with each other. In particular, repairing an unsafe behavior requires altering the NN's response in a \emph{local} region of the state space, but changing even a few neurons generally affects the \emph{global} response of the NN -- which could undo the initial safety guarantee supplied with the network. This tension is especially relevant for general deep NNs, and repairs realized on neurons in their latter layers. This is especially the case for repairing controllers, where the relationship between specific neurons and their importance to the overall safety properties is difficult to discern. As a result, there has been limited success in studying NN controller repair, especially for nonlinear systems.

In this paper, we exhibit an explicit algorithm that can repair a NN controller for a \emph{discrete-time, input-affine nonlinear} system. The cornerstone of our approach is to consider NN controllers of a specific architecture: in particular, the recently proposed Two-Level Lattice (TLL) NN architecture \cite{FerlezAReNAssuredReLU2020}. The TLL architecture has unique neuronal semantics, and those semantics greatly facilitate finding a balance between the local and global trade-offs inherent in NN repair. In particular, by assuming a TLL architecture, we can separate the problem of controller repair into two \emph{significantly decoupled} problems, one consisting of essentially only local considerations and one consisting of essentially only global ones.

\textbf{Related Work}: Repairing (or patching) NNs can be traced to the late 2000s. An early result on patching connected transfer learning and concept drift with patching \cite{Patch1};
%One of the early attempts in this direction was noticed by \cite{Patch1} who realized the %connection between transfer learning and concept drift with that of patching, furthermore
another result established fundamental requirements to apply classifier patching on NNs by using inner layers to learn a patch for concept drift in an image classifier network \cite{Patch2}.
% including those with practical applications on image recognition where instead of retraining the network to adequate concept drift, patching a NN takes advantage of the inner layers to learn a patch.\cite{Patch2}
Another approach based on a Satisfiability Modulo Theory (SMT) formulation of the repair problem was proposed by \cite{SRDM2019} where they changed the parameters of a classifier network to comply with a safety specification, i.e. where the designer knows exactly the subset of the input space to be classified.
% Furthermore, introducing a generalization of ReLU networks.
This prior work nonetheless is heuristic-based and so not guaranteed to produced desired results, which was noticed by \cite{minGoldberger} who  cast  the  problem  of patching (minimal repair) as a verification problem for NNs (including Deep ones). However, this work focused on a restricted version of the problem in which the changes in weights are limited to a single layer.
% however because DNN verification is NP-complete becomes exponentially more complex as the size of DNN increases.
% Like \cite{minGoldberger} work,
Finally, \cite{dong2020repairing} proposed a verification-based approach for repairing DNNs but not restricted to modifying the output; instead, proposed to identify and modify the most relevant neurons that causes the safety violation using gradient guidance.

% !TEX root = ./main.tex

\section{Preliminaries} % (fold)
\label{sec:preliminaries}

\subsection{Notation} % (fold)
\label{sub:notation}
We will denote the real numbers by $\mathbb{R}$. For an $(n \times m)$ matrix 
(or vector), $A$, we will use the notation $\llbracket A \rrbracket_{i,j}$ to 
denote the element in the $i^\text{th}$ row and $j^\text{th}$ column of $A$. 
Analogously, the notation $\llbracket A \rrbracket_{i,\cdot}$ will denote the 
$i^\text{th}$ row of $A$, and $\llbracket A \rrbracket_{\cdot, j}$ will denote 
the $j^\text{th}$ column of $A$; when $A$ is a vector instead of a matrix, both 
notations will return a scalar corresponding to the corresponding element in 
the vector. Let $\mathbf{0}_{n,m}$ be an $(n \times m)$ matrix of zeros. We 
will use bold parenthesis $\;\subarg{ \cdot }$ to delineate the arguments to a 
function that \emph{returns a function}.
% For example, given an $(m \times n)$ 
% matrix, $W$, (possibly with $m=1$) and an $(m \times 1)$ dimensional vector, 
% $b$, we define the linear function: $ \mathscr{L}^i \subarg{ W, b } : x \mapsto 
% \llbracket W x + b \rrbracket_i$ (that is $\mathscr{L}^i \subarg{ W, b }$ is 
% itself a function).
We use the functions $\mathtt{First}$ and $\mathtt{Last}$ to return the first 
and last elements of an ordered list (or a vector in $\mathbb{R}^n$). The 
function $\mathtt{Concat}$ concatenates two ordered lists, or two vectors in 
$\mathbb{R}^n$ and $\mathbb{R}^m$ along their (common) nontrivial dimension to 
get a third vector in $\mathbb{R}^{n+m}$.
% We will use an over-bar to indicate 
% (topological) closure of a set: i.e. $\overbar{A}$ is the closure of $A$. 
Finally, $B(x;\delta)$ denotes an open Euclidean ball centered at $x$ with 
radius $\delta$. The norm $\lVert \cdot \rVert$ will refer to the Euclidean 
norm.
% subsection notation (end)

\subsection{Dynamical Model}
In this paper, we will consider the general case of a discrete-time 
input-affine nonlinear system $\Sigma$ specified by:
\begin{equation} \label{eq:system_dynamics}
    \Sigma:
    \begin{cases}
    x_{i+1}=f(x_i)+g(x_i)u_i    
    \end{cases}
\end{equation}
where $x\in\mathbb{R}^n$ is the state, $u\in\mathbb{R}^m$ is the input. In 
addition, $f:\mathbb{R}^n \rightarrow \mathbb{R}^n$ and $g:\mathbb{R}^n 
\rightarrow \mathbb{R}^n$ are continuous and smooth functions of $x$.

\begin{definition}[Closed-loop Trajectory]
	Let $u : \mathbb{R}^n \rightarrow \mathbb{R}^m$. Then a \textbf{closed-loop 
	trajectory} of the system \eqref{eq:system_dynamics} under $u$, starting 
	from state $x_0$, will be denoted by the sequence 
	$\{\zeta_i^{x_0}(u)\}_{i=0}^\infty$. That is $\zeta_{i+1}^{x_0}(u) = 
	f(\zeta_{i}^{x_0}(u)) + g(\zeta_{i}^{x_0}(u)) \cdot u(\zeta_{i}^{x_0}(u))$ 
	and $\zeta_{0}^{x_0}(u) = x_0$.
\end{definition}

\begin{definition}[Workspace]
\label{def:workspace}
	We will assume that trajectories of \eqref{eq:system_dynamics} are confined 
	to a connected, compact workspace, $X_\textsf{ws}$ with non-empty interior, 
	of size $\text{ext}( X_\textsf{ws} ) \triangleq \sup_{x\in X_\textsf{ws}} 
	\lVert x \rVert$.
\end{definition}

\subsection{Neural Networks} % (fold)
\label{sub:neural_networks}
We will exclusively consider Rectified Linear Unit Neural Networks (ReLU NNs). 
A $K$-layer ReLU NN is specified by composing $K$ \emph{layer} functions, each 
of which may be either linear and nonlinear. A \emph{nonlinear} layer with 
$\mathfrak{i}$ inputs and $\mathfrak{o}$ outputs is specified by a 
$(\mathfrak{o} \times \mathfrak{i} )$ real-valued matrix of \emph{weights}, 
$W$, and a $(\mathfrak{o} \times 1)$ real-valued matrix of \emph{biases}, $b$ 
as follows: $L_{\theta} :  z \mapsto \max\{ W z + b, 0 \}$ with the $\max$ 
function taken element-wise, and $\theta \triangleq (W,b)$. A \emph{linear} 
layer is the same as a nonlinear layer, only it omits the nonlinearity 
$\max\{\cdot , 0\}$; such a layer will be indicated with a superscript 
\emph{lin}, e.g. $L^\text{lin}_{\theta}$. Thus, a $K$-layer ReLU NN function as 
above is specified by $K$ layer functions $\{L_{\theta^{(i)}} : i = 1, \dots, 
K\}$ that are \emph{composable}: i.e. they satisfy $\mathfrak{i}_{i} = 
\mathfrak{o}_{i-1}: i = 2, \dots, K$. We will annotate a ReLU function $\nn$ by 
a \emph{list of its parameters} $\Theta \triangleq$ $( \theta^{|1},$ $\dots,$ 
$\theta^{|K} )$\footnote{That is $\Theta$ is not the concatenation of the 
$\theta^{(i)}$ into a single large matrix, so it preserves information about 
the sizes of the constituent $\theta^{(i)}$.}.

The number of layers and the \emph{dimensions} of the matrices $\theta^{|i} = 
(\; W^{|i}, b^{|i}\; )$ specify the \emph{architecture} of the ReLU NN. 
Therefore, we will denote the \textbf{architecture} of the ReLU NN 
$\nn_{\Theta}$ by
% \begin{equation}
	$\text{Arch}(\Theta) \triangleq ( (n,\mathfrak{o}_{1}), (\mathfrak{i}_{2},\mathfrak{o}_{2}), \ldots, %(\mathfrak{i}_{K-1},\mathfrak{o}_{K-1}), 
		(\mathfrak{i}_{K}, m)).$
% \end{equation}

% subsection rectified_linear_unit_neural_networks (end)

\subsection{Special NN Operations} % (fold)
\label{sub:special_nn_operations}
% Here we define two different mechanisms for combining a pair of NNs in order to 
% obtain a third.
\begin{definition}[Sequential (Functional) Composition]
\label{def:functional_composition}
	Let $\nn_{\Theta_{\scriptscriptstyle 1}}$ and 
	$\nn_{\Theta_{\scriptscriptstyle 2}}$ be two NNs where 
	$\mathtt{Last}(\text{Arch}(\Theta_1)) = (\mathfrak{i}, \mathfrak{c})$ and 
	$\mathtt{First}(\text{Arch}(\Theta_2)) =  (\mathfrak{c}, \mathfrak{o})$. 
	Then the \textbf{functional composition} of 
	$\nn_{\Theta_{\scriptscriptstyle 1}}$ and $\nn_{\Theta_{\scriptscriptstyle 
	2}}$, i.e. $\nn_{\Theta_{\scriptscriptstyle 1}} \circ 
	\nn_{\Theta_{\scriptscriptstyle 2}}$, is a well defined NN, and can be 
	represented by the parameter list $\Theta_{1} \circ \Theta_{2} \triangleq 
	\mathtt{Concat}(\Theta_1, \Theta_2)$.
\end{definition}
\begin{definition}
	\label{def:parallel_composition}
	Let $\nn_{\Theta_{\scriptscriptstyle 1}}$ and 
	$\nn_{\Theta_{\scriptscriptstyle 2}}$ be two $K$-layer NNs with parameter 
	lists: $\Theta_i = ((W^{\scriptscriptstyle |1}_i, b^{\scriptscriptstyle 
	|1}_i), \dots, (W^{\scriptscriptstyle |K}_i, b^{\scriptscriptstyle |K}_i)), 
	i = 1,2$. Then the \textbf{parallel composition} of 
	$\nn_{\Theta_{\scriptscriptstyle 1}}$ and $\nn_{\Theta_{\scriptscriptstyle 
	2}}$ is a NN given by the parameter list
	\begin{equation}
		\Theta_{1} \parallel \Theta_{2} \triangleq \big(\negthinspace
			\left(
				\negthinspace
				\left[
					\begin{smallmatrix}
						W^{\scriptscriptstyle |1}_1 \\
						W^{\scriptscriptstyle |1}_2
					\end{smallmatrix}
				\right],
				\left[
					\begin{smallmatrix}
						b^{\scriptscriptstyle |1}_1 \\
						b^{\scriptscriptstyle |1}_2
					\end{smallmatrix}
				\right]
				\negthinspace
			\right),
			{\scriptstyle \dots},
			\left(
				\negthinspace
				\left[
					\begin{smallmatrix}
						W^{\scriptscriptstyle |K}_1 \\
						W^{\scriptscriptstyle |K}_2
					\end{smallmatrix}
				\right],
				\left[
					\begin{smallmatrix}
						b^{\scriptscriptstyle |K}_1 \\
						b^{\scriptscriptstyle |K}_2
					\end{smallmatrix}
				\right]
				\negthinspace
			\right)
		\negthinspace\big).
	\end{equation}
	That is $\Theta_{1} \negthickspace \parallel \negthickspace \Theta_{2}$ 
	accepts an input of the same size as (both) $\Theta_1$ and $\Theta_2$, but 
	has as many outputs as $\Theta_1$ and $\Theta_2$ combined.
\end{definition}

\begin{definition}[$n$-element $\min$/$\max$ NNs]
	\label{def:n-element_minmax_NN}
	An $n$\textbf{-element $\min$ network} is denoted by the parameter list 
	$\Theta_{\min_n}$. $\nn\subarg{\Theta_{\min_n}}: \mathbb{R}^n \rightarrow 
	\mathbb{R}$ such that $\nn\subarg{\Theta_{\min_n}}(x)$ is the the minimum 
	from among the components of $x$ (i.e. minimum according to the usual order 
	relation $<$ on $\mathbb{R}$). An $n$\textbf{-element $\max$ network} is 
	denoted by $\Theta_{\max_n}$, and functions analogously. These networks are 
	described in \cite{FerlezAReNAssuredReLU2020}.
	% These networks are 
	% defined in Appendix \ref{sec:appendix}.
\end{definition}

% subsection special_nn_operations (end)

\subsection{Two-Level-Lattice (TLL) Neural Networks} % (fold)
\label{sub:two_layer_lattice_neural_networks}
In this paper, we will be especially concerned with ReLU NNs that have the 
Two-Level Lattice (TLL) architecture, as introduced with the AReN algorithm in 
\cite{FerlezAReNAssuredReLU2020}. Thus we define a TLL NN as follows.
\begin{definition}[TLL NN {\cite[Theorem 2]{FerlezAReNAssuredReLU2020}}]
\label{def:scalar_tll}
A NN that maps $\mathbb{R}^n \rightarrow \mathbb{R}$ is said to be \textbf{TLL 
NN of size} $(N,M)$ if the size of its parameter list $\Xi_{\scriptscriptstyle 
N,M}$ can be characterized entirely by integers $N$ and $M$ as follows.
\begin{equation}
	\Xi_{N,M} \negthinspace \triangleq  \negthinspace
		\Theta_{\max_M} \negthinspace\negthinspace
	\circ \negthinspace
		\big(
			(\negthinspace\Theta_{\min_N} \negthinspace \circ \Theta_{S_1}\negthinspace) \negthinspace
			\parallel \negthinspace {\scriptstyle \dots} \negthinspace \parallel \negthinspace
			(\negthinspace\Theta_{\min_N} \negthinspace \circ \Theta_{S_M}\negthinspace)
		\big) \negthinspace
	\circ 
		\Theta_{\ell}
\end{equation}
where

\begin{itemize}
	\item $\Theta_\ell \triangleq ((W_\ell, b_\ell))$;

	\item  each $\Theta_{S_j}$ has the form $\Theta_{S_j} = \big( S_j, 
\mathbf{0}_{M,1} \big)$; and

	\item $S_j = \left[ \begin{smallmatrix} {\llbracket I_N 
		\rrbracket_{\iota_1, 
		\cdot}}\negthickspace\negthickspace\negthickspace^{^{\scriptscriptstyle\text{T}}} 
		& \; \dots \; & {\llbracket I_N \rrbracket_{\iota_N, 
		\cdot}}\negthickspace\negthickspace\negthickspace^{^{\scriptscriptstyle\text{T}}} 
		\end{smallmatrix} \right]^\text{T}$ for some sequence $\iota_k \in \{1, 
		\dots, N\}$, where $I_N$ is the $(N \times N)$ identity matrix. 
\end{itemize}
The matrices $\Theta_\ell$ will be referred to as the \textbf{linear function 
matrices} of $\Xi_{N,M}$. The matrices $\{ S_j | j = 1, \dots, M\}$ will be 
referred to as the \textbf{selector matrices} of $\Xi_{N,M}$. Each set $s_j 
\triangleq \{ k \in \{1, \dots, N\} | \exists \iota \in \{1, \dots, N\}. 
\llbracket S_j \rrbracket_{\iota,k} = 1 \}$ is said to be the selector set of 
$S_j$.

A \textbf{multi-output TLL NN} with range space $\mathbb{R}^m$ is defined using 
$m$ equally sized scalar TLL NNs. That is we denote such a network by 
$\Xi^{(m)}_{N,M}$, with each output component denoted by $\Xi^{i}_{N,M}$, $i = 
1, \dots, m$.
\end{definition}

\section{Problem Formulation} % (fold)
\label{sec:problem_formulation}
The main problem we consider in this paper is one of TLL NN \emph{repair}. In 
brief, we take as a starting point a TLL NN controller that is ``mostly'' 
correct in the sense that is provably safe under a specific set of 
circumstances (states); here we assume that safety entails avoiding a 
\emph{particular, \underline{fixed} subset of the state space}. However, we 
further suppose that this TLL NN controller induces some additional, 
\emph{unsafe} behavior of \eqref{eq:system_dynamics} that is explicitly 
observed, such as from a more expansive application of a model checker; of 
course this unsafe behavior necessarily occurs in states not covered by the 
original safety guarantee. The repair problem, then, is to ``repair'' the given 
TLL controller so that this additional unsafe behavior is made safe, while 
simultaneously preserving the original safety guarantees associated with the 
network.

The basis for the problem in this paper is thus a TLL NN controller that has 
been designed (or trained) to control \eqref{eq:system_dynamics} in a 
\emph{safe way}. In particular, we use the following definition to fix our 
notion of ``unsafe'' behavior for \eqref{eq:system_dynamics}.
\begin{definition}[Unsafe Operation of \eqref{eq:system_dynamics}]
\label{def:unsafe_states}
	Let $G_u$ be an $(\mathsf{K}_u \times n)$ real-valued matrix, and let $h_u$ 
	be an $(\mathsf{K}_u \times 1)$ real vector, which together define a set of 
	\textbf{unsafe states} $X_\textsf{unsafe} \triangleq \{ x \in \mathbb{R}^n 
	| G_u x \geq h_u\}$.
	% Any trajectory of \eqref{eq:system_dynamics} that 
	% enters $X_\textsf{unsafe}$ is called \textbf{unsafe}.
\end{definition}
\noindent Then, we mean that a TLL NN controller is safe with respect to 
\eqref{eq:system_dynamics} and $X_\textsf{unsafe}$ in the following sense.
\begin{definition}[Safe TLL NN Controller]
\label{def:safe_controller}
	% Let $G_s$ be an $(\mathsf{K}_s \times n)$ real-valued matrix, and let $h_s$ 
	% be an $(\mathsf{K}_s \times 1)$ real vector, which together define a set of 
	% states $X_\textsf{safe} \triangleq \{ x \in \mathbb{R}^n | G_s x \geq 
	% h_s\}$ 
	Let $X_\textsf{safe} \subset \mathbb{R}^n$ be a set of states
	such that $X_\textsf{safe} \cap X_\textsf{unsafe} = \emptyset$. Then a  TLL 
	NN controller $\mathfrak{u} \triangleq \nn\subarg{\Xi^{(m)}_{N,M}} : 
	\mathbb{R}^n \rightarrow \mathbb{R}^m$ is \textbf{safe} for 
	\eqref{eq:system_dynamics} \textbf{on horizon} $T$ (with respect to 
	$X_\textsf{safe}$ and $X_\textsf{unsafe}$) if:
	\begin{equation}
		\forall x_0 \negthinspace\in\negthinspace X_\textsf{safe} , i \negthinspace\in\negthinspace \{1, {\scriptstyle \dots} , T\} .
		\big(
		\zeta_i^{x_0}\negthinspace(\nn\subarg{ \Xi^{(m)}_{N,M} })
		\negthinspace \not\in \negthinspace
		X_\textsf{unsafe}
		\big).
	\end{equation}
	That is $\nn\subarg{\Xi^{(m)}_{N,M}}$ is safe (w.r.t. $X_\textsf{safe}$) if 
	all of its length-$T$ trajectories starting in $X_\textsf{safe}$ avoid the 
	unsafe states $X_\textsf{unsafe}$.
\end{definition}
\noindent The design of safe controllers in the sense of Definition 
\ref{def:safe_controller} has been considered in a number of contexts; see e.g. 
\cite{TranNNVNeuralNetwork2020}. Often this design procedure involves training 
the NN using data collected from an expert, and verifying the result using one 
of many available NN verifiers \cite{TranNNVNeuralNetwork2020}.

However, as noted above, we further suppose that a given TLL NN which is safe 
in the sense of Definition \ref{def:safe_controller} nevertheless has some 
\emph{unsafe} behavior for states that lie outside $X_\textsf{safe}$. In 
particular, we suppose that a model checker (for example) provides to us a 
\emph{counterexample} (or witness) to unsafe operation of 
\eqref{eq:system_dynamics}.

\begin{definition}[Counterexample to Safe Operation of \eqref{eq:system_dynamics}]
\label{def:counterexample}
	Let $X_\textsf{safe} \subset \mathbb{R}^n$, and let $\mathfrak{u} 
	\triangleq \nn\subarg{\Xi^{(m)}_{N,M}}$ be a TLL controller that is safe for 
	\eqref{eq:system_dynamics} on horizon $T$ w.r.t $X_\textsf{safe}$ and 
	$X_\textsf{unsafe}$. A \textbf{counter example to the safe operation of} 
	\eqref{eq:system_dynamics} is a state $x_\textsf{c.e.} \not\in 
	X_\textsf{safe}$ such that
	\begin{equation}
		f(x_\textsf{c.e.}) + g(x_\textsf{c.e.}) \cdot \mathfrak{u}(x_\textsf{c.e.}) = \zeta_1^{x_\textsf{c.e.}}(\mathfrak{u}) \in X_\textsf{unsafe}.
	\end{equation}
	That is starting \eqref{eq:system_dynamics} in $x_\textsf{c.e.}$ results in 
	an unsafe state in the next time step.
\end{definition}

We can now state the main problem of this paper.

\begin{problem}
\label{prob:main_problem}
	Let dynamics \eqref{eq:system_dynamics} be given, and assume its 
	trajectories are confined to compact subset of states, $X_\textsf{ws}$ (see 
	Definition \ref{def:workspace}). Also, let $X_\textsf{unsafe} \subset 
	X_\textsf{ws}$ be a specified set of unsafe states for 
	\eqref{eq:system_dynamics}, as in Definition \ref{def:unsafe_states}. 
	Furthermore, let $\mathfrak{u} = \nn\subarg{\Xi^{(m)}_{N,M}}$ be a TLL NN 
	controller for \eqref{eq:system_dynamics} that is safe on horizon $T$ with 
	respect to a set of states $X_\textsf{safe} \subset X_\textsf{ws}$ (see 
	Definition \ref{def:safe_controller}), and let $x_\textsf{c.e.}$ be a 
	counterexample to safety in the sense of Definition 
	\ref{def:counterexample}.

	Then the \textbf{TLL repair problem} is to obtain a new TLL controller 
	$\overbar{\mathfrak{u}} = \nn\subarg{\overbar{\Xi}^{(m)}_{N,M}}$ with the 
	following properties:

	\begin{enumerate}[label={(\itshape \roman*)}]
		\item $\overbar{\mathfrak{u}}$ is also safe on horizon $T$ with 
			respect to $X_\textsf{safe}$;

		\item the trajectory 
			$\zeta_1^{x_\textsf{c.e.}}(\overbar{\mathfrak{u}})$ is safe -- i.e. 
			the counterexample $x_\textsf{c.e.}$ is ``repaired'';

		\item $\overbar{\Xi}^{(m)}_{N,M}$ and ${\Xi}^{(m)}_{N,M}$ share a 
			common architecture (as implied by their identical architectural 
			parameters); and 

		\item the selector matrices of $\overbar{\Xi}^{(m)}_{N,M}$ and 
			${\Xi}^{(m)}_{N,M}$ are identical -- i.e. $\overbar{S}_k = S_k$ for 
			$k = 1, \dots, M$; and 

		\item $\lVert \overbar{W}_\ell - W_\ell \rVert_2 + \lVert 
			\overbar{b}_\ell - b_\ell \rVert_2$ is minimized.
	\end{enumerate}
\end{problem}

In particular, \emph{iii)}, \emph{iv)} and \emph{v)} justify the designation of 
this problem as one of ``repair''. That is the repair problem is to fix the 
counterexample while keeping the network as close as possible to the original 
network under consideration. \textbf{Note:} the formulation of Problem 
\ref{prob:main_problem} only allows repair by means of altering the 
\emph{linear layers} of $\Xi^{(m)}_{N,M}$; c.f. \emph{(iii)} and \emph{(iv)}.

\section{Framework}
\label{sec:framework}

The TLL NN repair problem described in Problem \ref{prob:main_problem} is 
challenging because it has two main objectives, which are at odds with each 
other. In particular, repairing a counterexample requires altering the NN's 
response in a \emph{local} region of the state space, but changing even a few 
neurons generally affects the \emph{global} response of the NN -- which could 
undo the initial safety guarantee supplied with the network. This tension is 
especially relevant for general deep NNs, and repairs realized on neurons in 
their latter layers. It is for this reason that we posed Problem 
\ref{prob:main_problem} in terms of TLL NNs: our approach will be to use the 
unique semantics of TLL NNs to balance the trade-offs between 
\textbf{\underline{local} NN alteration to repair the defective controller}  
and \textbf{\underline{global} NN alteration % consider 'preservation' here
to ensure that the repaired controller activates at the counterexample}. 
Moreover, locally repairing the defective controller at $x_\textsf{c.e.}$ 
entails a further trade off between two competing objectives of its own: 
actually repairing the counterexample -- Problem 
\ref{prob:main_problem}\emph{(ii)} -- without causing a violation of the 
original safety guarantee for $X_\textsf{safe}$ -- i.e. Problem 
\ref{prob:main_problem}\emph{(i)}. Likewise, global alteration of the TLL to 
ensure correct activation of our repairs will entail its own trade-off: the 
alterations necessary to achieve the correct activation will also have to be 
made without sacrificing the safety guarantee for $X_\textsf{safe}$ -- i.e. 
Problem \ref{prob:main_problem}\emph{(i)}.

We devote the remainder of this section to two crucial subsections, one for 
each side of this local/global dichotomy. Our goal in these two subsections is 
to describe \textbf{constraints} on a TLL controller that are 
\textbf{sufficient} to ensure that it accomplishes the repair described in 
Problem \ref{prob:main_problem}. Thus, the results in this section should be 
seen as optimization constraints around which we can build our algorithm to 
solve Problem \ref{prob:main_problem}. The algorithmic details and formalism 
are presented in Section \ref{sec:main_algorithm}.

\subsection{Local TLL Repair} % (fold)
\label{sub:local_repair}

We first consider in isolation the problem of repairing the TLL controller in 
the vicinity of the counterexample $x_\textsf{c.e.}$, but under the assumption 
that the altered controller will remain the active there. The problem of 
actually guaranteeing that this is the case will be considered in the 
subsequent section. Thus, we proceed with the repair by \textbf{establishing 
constraints} on the alterations of those parameters in the TLL controller 
associated with the affine controller instantiated at and around the state 
$x_\textsf{c.e.}$. To be consistent with the 
literature, we will refer to any individual affine function instantiated by a 
NN as one of its \emph{local linear functions}.

\begin{definition}[Local Linear Function]
\label{def:local_linear_function}
	Let $\mathsf{f} : \mathbb{R}^n \rightarrow \mathbb{R}$ be CPWA. Then a 
	\textbf{local linear function of} $\mathsf{f}$ is a linear function $\ell : 
	\mathbb{R}^n \rightarrow \mathbb{R}$ if there exists an open set 
	$\mathfrak{O}$ such that $\ell(x) = \mathsf{f}(x)$ for all $x\in 
	\mathfrak{O}$.
\end{definition}

The unique semantics of TLL NNs makes them especially well suited to this local 
repair task because in a TLL NN, its local linear functions appear directly as 
neuronal parameters. In particular, all of the local linear functions of a TLL 
NN are described \emph{directly} by parameters in its linear layer; i.e. 
$\Theta_\ell = (W_\ell, b_\ell)$ for scalar TLL NNs or $\Theta^\kappa_\ell = 
(W^\kappa_\ell, b^\kappa_\ell)$ for the $\kappa^\text{th}$ output of a 
multi-output TLL (see Definition \ref{def:scalar_tll}). This follows as a 
corollary of the following relatively straightforward proposition, borrowed from 
\cite{FerlezBoundingComplexityFormally2020}:

\begin{proposition}[{\cite[Proposition 3]{FerlezBoundingComplexityFormally2020}}]
\label{prop:local_lin_fns_params}
	Let $\Xi_{N,M}$ be a scalar TLL NN with linear function matrices 
	$\Theta_\ell = (W_\ell, b_\ell) $. Then every local linear function of 
	$\nn\subarg{\Xi_{N,M}}$ is exactly equal to $\ell_i : x \mapsto \llbracket 
	W_\ell x + b_\ell \rrbracket_{i,\cdot}$ for some $i \in \{1, \dots, N\}$.

	Similarly, let $\Xi^{(m)}_{N,M}$ be a multi-output TLL, and let $\ell$ be 
	any local linear function of $\nn\subarg{\Xi^{(m)}_{N,M}}$. Then for each 
	$\kappa \in \{1, \dots, m\}$, the $\kappa^\text{th}$ component of $\ell$ 
	satisfies $\llbracket \ell \rrbracket_{\kappa, \cdot} = x \mapsto 
	\llbracket W^\kappa_\ell x + b^\kappa_\ell \rrbracket_{{i_\kappa},\cdot}$ 
	for some $i_\kappa \in \{1, \dots, N\}$.
\end{proposition}

\begin{corollary}
\label{cor:active_local_linear_fn}
	Let $\Xi^{(m)}_{N,M}$ be a TLL over domain $\mathbb{R}^n$, and let 
	$x_\textsf{c.e.} \in \mathbb{R}^n$. Then there exist $m$ integers 
	$\text{act}_k \in \{1, \dots, N\}$ for $k = 1, \dots, m$ and a closed, 
	connected set with non-empty interior, $R_\text{a} \subset \mathbb{R}^n$ 
	such that

	\begin{itemize}
		\item $x_\textsf{c.e.} \in R_\text{a}$; and 

		\item $\llbracket \nn\subarg{\Xi^{(m)}_{N,M}} \rrbracket_{k} = x 
			\mapsto \llbracket W^k_\ell x + b^k \rrbracket_{\text{act}_k}$ on 
			the set $R_\text{a}$.
	\end{itemize}
\end{corollary}
\noindent Corollary \ref{cor:active_local_linear_fn} is actually a strong 
statement: it indicates that in a TLL, each local linear function is described 
directly by \emph{its own} linear-function-layer parameters and those 
parameters describe \emph{only} that local linear function.

Thus, as a consequence of Corollary \ref{cor:active_local_linear_fn}, 
``repairing'' the problematic local controller (local linear function) of the 
TLL controller in Problem \ref{prob:main_problem} involves the following steps:
\begin{enumerate}
	\item identify which of the local linear functions is realized by the 
		TLL controller at $x_\textsf{c.e.}$ -- i.e. identifying the indices of 
		the active local linear function at $x_\textsf{c.e.}$ viz. indices 
		$\text{act}_\kappa \in \{1, \dots, N\}$ for each output $\kappa$ as in 
		Corollary \ref{cor:active_local_linear_fn};

	\item \underline{\emph{establish constraints}} on the parameters of that 
		local linear function so as to ensure repair of the counterexample; 
		i.e. altering the elements of the rows $\llbracket W_\ell^\kappa 
		\rrbracket_{\text{act}_\kappa, \cdot}$ and $\llbracket b_\ell^\kappa 
		\rrbracket_{\text{act}_\kappa}$ for each output $\kappa$ such that the 
		resulting linear controller repairs the counterexample as in Problem 
		\ref{prob:main_problem}\emph{(ii)}; and 

	\item \underline{\emph{establish constraints}} to ensure the repaired 
		parameters do not induce a violation of the safety constraint for the 
		guaranteed set of safe states, $X_\textsf{safe}$, as 
		in Problem \ref{prob:main_problem}\emph{(i)}.
\end{enumerate}
We consider these three steps in sequence as follows.
% Local linear function of a CPWA are defined as follows; in the case of a TLL 
% controller, they correspond to all of the linear controllers that the TLL 
% controller can implement.  

\subsubsection{Identifying the Active Controller at $x_\textsf{c.e.}$} % (fold)
\label{ssub:identifying_the_active_controller_at_}

% We begin a formal consideration of this problem by defining the local linear 
% function(s) of a CPWA, and then restate a simple, but crucial, proposition that 
% connects them to the semantics of TLL NN parameters.

From Corollary \ref{cor:active_local_linear_fn}, \emph{all} of the possible 
linear controllers that a TLL controller realizes are exposed directly in the 
parameters of its linear layer matrices, $\Theta_\ell^\kappa$. Crucially for 
the repair problem, once the active controller at $x_\textsf{c.e.}$ has been 
identified, the TLL parameters responsible for that controller immediately 
evident. This is the starting point for our repair process.

Since a TLL consists of two levels of lattice operations, it is straightforward 
to identify which of these affine functions is in fact active at 
$x_\textsf{c.e.}$; for a given output, $\kappa$, this is can be done by 
evaluating $W_\ell^\kappa x_\textsf{c.e.} + b_\ell^\kappa$ and comparing the 
components thereof according to the selector sets associated with the TLL 
controller. That is the index of the active controller for output $\kappa$, 
denoted by $\text{act}_\kappa$, is determined by the following two expressions:
\begin{align}
	\mu^\kappa_k &\triangleq \arg \min_{i \in S^\kappa_k} \llbracket W^\kappa_\ell x_\textsf{c.e.} + b^\kappa_\ell \rrbracket_{i} \\
	\text{act}_\kappa &\triangleq \hspace{10pt} \arg \hspace{-22pt} \max_{j \in \{\mu^\kappa_k | k = 1, \dots, M\}} \llbracket W^\kappa_\ell x_\textsf{c.e.} + b^\kappa_\ell \rrbracket_{j}
\end{align}
These expressions mirror the computations that define a TLL network, as 
described in Definition \ref{def:scalar_tll}; the only difference is that 
$\max$ and $\min$ are replaced by $\arg \max$ and $\arg \min$, respectively, so 
as to retrieve the index of interest instead of the network's 
output.
% \footnote{It is possible of course that more than one local linear 
% function is active at $x_\textsf{c.e.}$: this corresponds to a point at which 
% one of the output components of the TLL NN switches from one of its local 
% linear function to another. We will handle this eventuality in the global 
% stasis aspect.}

% subsubsection identifying_the_active_controller_at_ (end)

\subsubsection{Repairing the Affine Controller at $x_\textsf{c.e.}$} % (fold)
\label{ssub:repairing_the_affine_controller_at_}

Given the result of Corollary \ref{cor:active_local_linear_fn}, the parameters 
of the network that result in a problematic controller at $x_\textsf{c.e.}$ are 
readily apparent.  Moreover, since these parameters are obviously in the linear 
layer of the original TLL, they are alterable under the requirement in Problem 
\ref{prob:main_problem} that only linear-layer parameters are permitted to be 
used for repair. Thus, in the current context, local repair entails simply 
correcting the elements of the matrices $\llbracket W^k_\ell 
\rrbracket_{\text{act}_k}$ and $\llbracket b^k_\ell \rrbracket_{\text{act}_k}$. 
It is thus clear that a ``repaired'' controller should satisfy
\begin{equation}
\label{eq:repair_constraint}
	f(x_\textsf{c.e.}) + g(x_\textsf{c.e.})
	\left[
	\begin{smallmatrix}
		\llbracket W^1_\ell x_\textsf{c.e.} + b^1_\ell \rrbracket_{\text{act}_1} \\
		\vdots \\
		\llbracket W^m_\ell x_\textsf{c.e.} + b^m_\ell \rrbracket_{\text{act}_m} 
	\end{smallmatrix}
	\right]
	\not\in
	X_\textsf{unsafe}.
\end{equation}

Then \eqref{eq:repair_constraint} represents a \emph{linear constraint} in the 
local controller to be repaired, and this constraint imposes the repair 
property in Problem \ref{prob:main_problem}\emph{(ii)}. That is provided that 
the repaired controller described by $\{\text{act}_\kappa\}$ \emph{remains 
active} at the counterexample; as noted, we consider this problem in the global 
stasis condition subsequently.
% subsubsection altering_the_affine_controller_at_ (end)

\subsubsection{Preserving the Initial Safety Condition with the Repaired Controller} % (fold)
\label{ssub:preserving_the_initial_safety_condition}
One unique aspect of the TLL NN architecture is that affine functions defined 
in its linear layer can be \emph{reused} across regions of its input space. In 
particular, the controller associated with the parameters we repaired in the 
previous step -- i.e. the indices $\{\text{act}_\kappa\}$ of the linear layer 
matrices -- may likewise be activated in or around $X_\textsf{safe}$. The fact 
that we \emph{altered} these controller parameters thus means that trajectories 
emanating from $X_\textsf{safe}$ may be affected in turn by our repair efforts: 
that is the repairs we made to the controller to address Problem 
\ref{prob:main_problem}\emph{(ii)} may simultaneously alter the TLL in a way 
that \textbf{undoes} the requirement in Problem 
\ref{prob:main_problem}\emph{(i)} -- i.e. the initial safety guarantee on 
$X_\textsf{safe}$ and $\nn\subarg{\Xi^{(m)}_{N,M}}$. Thus, local repair of the 
problematic controller must account for this safety property, too.

We accomplish this by bounding the reach set of \eqref{eq:system_dynamics} for  
initial conditions in $X_\textsf{safe}$, and for this we employ the usual 
strategy of bounding the relevant Lipschitz constants. Naturally, since the TLL 
controller is a CPWA controller operated in closed loop, these bounds will also 
incorporate the size of the TLL controller parameters $\lVert \llbracket 
W^\kappa_\ell \rrbracket_i \rVert$ and $\lVert \llbracket b^\kappa_\ell 
\rrbracket_i \rVert$ for $\kappa \in \{1, \dots, m\}$ and $i \in \{1, \dots, N 
\}$.

In general, however, we have the following proposition.
\begin{proposition}
\label{prop:reach_set_bound}
	Consider system dynamics \eqref{eq:system_dynamics}, and suppose that the 
	state $x$ is confined to known compact workspace, $X_\textsf{ws}$ (see 
	Definition \ref{def:workspace}). Also, let $T$ be the integer time horizon 
	from Definition \ref{def:safe_controller}. Finally, assume that a 
	closed-loop CPWA $\Psi : \mathbb{R}^n \rightarrow \mathbb{R}^m$ is applied 
	to \eqref{eq:system_dynamics}, and that $\Psi$ has local linear functions 
	$\mathcal{L}_\Psi = \{ x \mapsto w_k x + b_k | k = 1, \dots, N\}$.

	Moreover, define the function $\beta$ as 
	\begin{align}
		&\beta(\lVert w \rVert, \lVert b \rVert) \triangleq 
			\sup_{x_0 \in X_\textsf{safe}} 
			\Big(
			\lVert f(x_0) - x_0 \rVert + \notag \\
			&\qquad\quad \lVert g(x_0) \rVert \cdot \lVert w \rVert \cdot \text{ext}(X_\textsf{ws})
			+ \lVert g(x_0) \rVert \cdot \lVert b \rVert
			\Big)
	\end{align}
	and in turn define
	\begin{equation}
		\beta_\text{max}(\Psi) \triangleq
			\beta\big(
			\max_{
				w \in \{w_k | k = 1, \dots, N\}
			}
			\lVert w \rVert,
			\max_{
				b \in \{b_k | k = 1, \dots, N\}
			}
			\lVert b \rVert
			\big).
	\end{equation}
	Finally, define the function $L$ as in \eqref{eq:L},
		\begin{figure*}[!b]
			% IEEE uses as a separator
			\hrulefill
			% ensure that we have normalsize text
			\normalsize
			% Store the current equation number.
			\setcounter{MYtempeqncnt}{\value{equation}}
			% Set the equation number to one less than the one
			% desired for the first equation here.
			% The value here will have to changed if equations
			% are added or removed prior to the place these
			% equations are referenced in the main text.
			\setcounter{equation}{10}
			\begin{equation}\label{eq:L}
				L(\lVert w \rVert, \lVert b \rVert) \triangleq 
						L_f + 
							L_g 
							\cdot
							\sup_{x_0 \in X_\textsf{safe}} 
							% \max_{
							% 	w^\kappa_j := \llbracket W^\kappa_\ell \rrbracket_j
							% } 
							\lVert w \rVert \cdot \lVert x_0 \rVert
							+ 
							\sup_{x_0 \in X_\textsf{safe}} 
							% \max_{
							% 	w^\kappa_j := \llbracket W^\kappa_\ell \rrbracket_j
							% }
							\lVert w \rVert \cdot \lVert g(x_0) \rVert
						+
							L_g \cdot 
							% \max_{
							% 	b^\kappa_j := \llbracket b^\kappa_\ell \rrbracket_j
							% }
							\lVert b \rVert
			\end{equation}
			% Restore the current equation number.
			\setcounter{equation}{\value{MYtempeqncnt}}
			\addtocounter{equation}{1}
			% \setcounter{equation}{11}
			% The spacer can be tweaked to stop underfull vboxes.
			\vspace*{4pt}
		\end{figure*}
	and in turn define
	\begin{equation}
		L_\text{max}(\Psi) \triangleq L\big( 
				\max_{
				w \in \{w_k | k = 1, \dots, N\}
				}
				\lVert w \rVert,
				\max_{
				b \in \{b_k | k = 1, \dots, N\}
				}
				\lVert b \rVert
			\big).
	\end{equation}

	Then for all $x_0 \in X_\textsf{safe}$, $i \in \{1, \dots, T\}$, we have:
	\begin{equation}\label{eq:main_bound_prop_conclusion}
			\lVert \zeta_T^{x_0}(\Psi) - x_0 \lVert
			\leq
			\beta_\text{max}(\Psi) \cdot \sum_{k=0}^T   { L_\text{max}(\Psi) }^k.
	\end{equation}
\end{proposition}
\noindent The proof of Proposition \ref{prop:reach_set_bound} is in Appendix \ref{sec:appendix} of \cite{SantaCruzTechReport}.

Proposition \ref{prop:reach_set_bound} bounds the size of the reach set for 
\eqref{eq:system_dynamics} in terms of an arbitrary CPWA controller, $\Psi$, 
when the system is started from $X_\textsf{safe}$. This proposition is 
naturally applied in order to find bounds for safety with respect to the 
\emph{unsafe} region $X_\textsf{unsafe}$ as follows.

\begin{proposition}
\label{prop:reach_set_bound_safe}
	Let $T$, $X_\textsf{ws}$, $\Psi$ and $\mathcal{L}_\Psi$ be as in 
	Proposition \ref{prop:reach_set_bound}, and let $\beta_\text{max}$ and 
	$L_\text{max}$ be two constants s.t. for all $\delta x 
	\in \mathbb{R}^n$
	\begin{multline}\label{eq:safe_reach_constants}
		\lVert \delta x \rVert \leq \beta_\text{max} \cdot \sum_{k=0}^T {L_\text{max}}^k \\
		\implies \forall x_0 \in X_\textsf{safe} . \big( x_0 + \delta x \not\in X_\textsf{unsafe} \big)
	\end{multline}

	If $\beta_\text{max}(\Psi) \leq \beta_\text{max}$ and $L_\text{max}(\Psi) 
	\leq L_\text{max}$, then trajectories of \eqref{eq:system_dynamics} under 
	closed loop controller $\Psi$ are safe in the sense that
	\begin{equation}
		\forall x_0 \in X_\textsf{safe} \forall i \in \{1, \dots, T\} \;.\;
		\zeta^{x_0}_t(\Psi) \not\in X_\textsf{unsafe}.
	\end{equation}
\end{proposition}
The proof of is a more or less straightforward application of Proposition 
\ref{prop:reach_set_bound}, and so can be found in Appendix \ref{sec:appendix} of \cite{SantaCruzTechReport}.

In particular, Proposition \ref{prop:reach_set_bound_safe} states that if we 
find constants $\beta_\text{max}$ and $L_\text{max}$ that satisfy 
\eqref{eq:safe_reach_constants}, then we have a way to bound the parameters of 
any CPWA controller (via $\beta$ and $L$) so that that controller is safe in 
closed loop. This translates to conditions that our 
repaired controller must satisfy in order to preserve the safety property 
required in Problem \ref{prob:main_problem}\emph{(i)}.

Formally, this entails particularizing Proposition 
\ref{prop:reach_set_bound} and \ref{prop:reach_set_bound_safe} to 
the TLL controllers associated with the repair problem.
\begin{corollary}
\label{cor:repaired_tll_reach_set_bound}
	Again consider system \eqref{eq:system_dynamics} confined to workspace 
	$X_\textsf{ws}$ as before.  Also, let $\beta_\text{max}$ and 
	$L_\text{max}$ be such that they satisfy the assumptions of Proposition 
	\ref{prop:reach_set_bound_safe}, viz. \eqref{eq:safe_reach_constants}. 

	Now, let $\Xi^{(m)}_{N,M}$ be the TLL controller as given in Problem 
	\ref{prob:main_problem}, and let $\Theta_{\ell}^\kappa = 
	(W^\kappa_\ell,b^\kappa_\ell)$ be its linear layer matrices for outputs 
	$\kappa = 1, \dots, m$ as usual. For this controller, define the following 
	two quantities:
	\begin{align}
		\Omega_W &\triangleq \max_{
				w \in \cup_{\kappa = 1}^{m} \{ \llbracket W^\kappa_\ell \rrbracket_j | j = 1, \dots, N\}
			}
			\lVert w \rVert \\
		\Omega_b &\triangleq \max_{
				b \in \cup_{\kappa = 1}^{m} \{ \llbracket b^\kappa_\ell \rrbracket_j | j = 1, \dots, N\}
			}
			\lVert b \rVert
	\end{align}
	so that $\beta_\text{max}(\Xi^{(m)}_{N,M}) = \beta(\Omega_W, \Omega_b)$ and 
	$L_\text{max}(\Xi^{(m)}_{N,M}) = L(\Omega_W, \Omega_b)$. Finally, let 
	indices $\{\text{act}_\kappa\}_{\kappa=1}^m$ specify the active local 
	linear functions of $\Xi^{(m)}_{N,M}$ that are to be repaired, as described 
	in Subsection \ref{ssub:identifying_the_active_controller_at_} and 
	\ref{ssub:repairing_the_affine_controller_at_}. Let 
	$\overbar{w}_{\text{act}_\kappa}^\kappa$ and 
	$\overbar{b}_{\text{act}_\kappa}^\kappa$ be any repaired values of 
	$\llbracket W^\kappa_\ell \rrbracket_{\text{act}_\kappa, \cdot}$ and 
	$\llbracket b^\kappa_\ell \rrbracket_{\text{act}_\kappa}$, respectively.

	If the following four conditions are satisfied
	\begin{align}
		\beta(\lVert 
				\overbar{w}_{\text{act}_\kappa}^\kappa \rVert, \lVert 
				\overbar{b}_{\text{act}_\kappa}^\kappa \rVert) 
		&\leq 
		\beta_\text{max} \label{eq:repair_beta_constraint}\\
		\beta_\text{max}(\Xi^{(m)}_{N,M})
		&\leq
		\beta_\text{max} \label{eq:original_beta_constraint} \\
		L(\lVert 
				\overbar{w}_{\text{act}_\kappa}^\kappa \rVert, \lVert 
				\overbar{b}_{\text{act}_\kappa}^\kappa \rVert)
		&\leq L_\text{max} \label{eq:repair_L_constraint} \\
		L_\text{max}(\Xi^{(m)}_{N,M}) &\leq L_\text{max}
		\label{eq:original_L_constraint}
	\end{align}
	then the following hold for all $x_0 \in X_\textsf{safe}$:
	\begin{equation}
		\lVert \zeta_T^{x_0}(\overbar{\Xi}^{(m)}_{N,M}) - x_0 \lVert
			\leq
			\beta_\text{max} \cdot \sum_{k=0}^T   { L_\text{max} }^k 
			\label{eq:cor_bounded_trajectories}
	\end{equation}
	and hence
	\begin{equation}
		\forall i \in \{1, \dots, T\} \;.\; \zeta_i^{x_0}(\overbar{\Xi}^{(m)}_{N,M}) \not\in X_\textsf{unsafe}.
			\label{eq:cor_safe_trajectories}
	\end{equation}
\end{corollary}
The proof of Corollary \ref{cor:repaired_tll_reach_set_bound} 
is in Appendix \ref{sec:appendix} of \cite{SantaCruzTechReport}.

The conclusion \eqref{eq:cor_bounded_trajectories} of Corollary 
\ref{cor:repaired_tll_reach_set_bound} should be interpreted as follows: the 
bound on the reach set of the repaired controller, $\overbar{\Xi}^{(m)}_{N,M}$, 
is no worse than the bound on the reach set of the original TLL controller 
given in Problem \ref{prob:main_problem}. Hence, by the assumptions borrowed 
from Proposition \ref{prop:reach_set_bound_safe}, conclusion 
\eqref{eq:cor_safe_trajectories} of Corollary 
\ref{cor:repaired_tll_reach_set_bound} indicates that \textbf{the repaired 
controller $\overbar{\Xi}^{(m)}_{N,M}$ remains safe in the sense of }Problem 
\ref{prob:main_problem}\emph{(i)} -- i.e. closed-loop trajectories emanating 
from $X_\textsf{safe}$ remain safe on horizon $T$.

For the subsequent development of our algorithm, 
\eqref{eq:repair_beta_constraint} and \eqref{eq:repair_L_constraint} will play 
the crucial role of ensuring that the repaired controller respects the 
guarantee of Problem \ref{prob:main_problem}\emph{(i)}.
% subsubsection preserving_the_initial_safety_condition (end)

% subsection local_repair (end)

\subsection{Global TLL Alteration for Repaired Controller Activation} % (fold)
\label{sub:global_nn_stasis}
In the context of local repair, we identified the local linear function 
instantiated by the TLL controller, and repaired the parameters associated with 
that particular function -- i.e. the repairs were affected on a particular, 
indexed row of $W_\ell^\kappa$ and $b_\ell^\kappa$. We then proceeded under the 
assumption that the affine function \emph{at that index} would remain active in 
the output of the TLL network at the counterexample, \emph{even 
\underline{after}} altering its parameters. Unfortunately, this is not case in 
a TLL network per se, since the value of each local linear function at a point 
interacts with the selector matrices (see Definition \ref{def:scalar_tll}) to 
determine whether it is active or not. In other words, changing the parameters 
of a particular indexed local linear function in a TLL will change its output 
value at any given point (in general), and hence also the region on which said 
indexed local linear function is \emph{active}. Analogous to the local 
alteration consider before, we thus need to \textbf{devise \underline{global} 
constraints sufficient to enforce the activation of the repaired controller at} 
$x_\textsf{c.e.}$.

This observation is manifest in the computation structure that defines a TLL 
NN: a particular affine function is active in the output of the TLL if and only 
if it is active in the output of one of the $\min$ networks (see Definition 
\ref{def:scalar_tll}), and the output of that same $\min$ network exceeds the 
output of all others, thereby being active at the output of the final $\max$ 
network (again, see Definition \ref{def:scalar_tll}). Thus, ensuring that a 
particular, indexed local linear function is active at the output of a TLL 
entails ensuring that that function 
\begin{enumerate}[label={(\itshape \alph*)}]
	\item appears at the output of one of the $\min$ networks; and

	\item appears at the output of the $\max$ network, by exceeding the 
		outputs of all the \emph{other} $\min$ networks.
\end{enumerate}
Notably, this sequence also suggests a mechanism for meeting the task at hand: 
ensuring that the repaired controller remains active at the counter example.

Formally, we have the following proposition.
\begin{proposition}\label{prop:activation_constraint}
	Let $\Xi^{(m)}_{N,M}$ be a TLL NN over $\mathbb{R}^n$ with output-component 
	linear function matrices $\Theta^\kappa_\ell = (W^\kappa_\ell, 
	b^\kappa_\ell)$ as usual, and let $x_\textsf{c.e.} \in \mathbb{R}^n$.

	Then the index $\text{act}_\kappa \in \{1, \dots, N\}$ denote the local 
	linear function that is active at $x_\textsf{c.e.}$ for output $\kappa$, as 
	described in Corollary \ref{cor:active_local_linear_fn}, if and only if 
	there exists index $\text{sel}_\kappa \in \{1, \dots, M\}$ such that
	\begin{enumerate}[label={(\itshape \roman*)}]
		\item for all $i \in S^\kappa_{\text{sel}_\kappa}$ and any $x \in 
			R_\text{a}$,
			\begin{equation}
				\llbracket W^\kappa_\ell x + b^\kappa_\ell \rrbracket_{\text{act}_\kappa, \cdot}
				\leq
				\llbracket W^\kappa_\ell x + b^\kappa_\ell \rrbracket_{i, \cdot}
			\end{equation}
			i.e. the active local linear function ``survives'' the $\min$ 
			network associated with selector set 
			$S^\kappa_{\text{sel}_\kappa}$; and

		\item for all $j \in \{1, \dots, M\} \backslash 
			\{\text{sel}_\kappa\}$ there exists an index $\iota_j^\kappa \in 
			\{1, \dots, N\}$ s.t. for all $x\in R_\text{a}$
			\begin{equation}
				\llbracket W^\kappa_\ell x + b^\kappa_\ell \rrbracket_{\iota_j^\kappa, \cdot}
				\leq
				\llbracket W^\kappa_\ell x + b^\kappa_\ell \rrbracket_{\text{act}_\kappa, \cdot}
			\end{equation}
			i.e. the active local linear function ``survives'' the $\max$ 
			network of output $\kappa$ by exceeding the output of all of the 
			other $\min$ networks.
	\end{enumerate}
\end{proposition}
This proposition follows calculations similar to those mentioned before; the 
 proof is in Appendix \ref{sec:appendix} of \cite{SantaCruzTechReport}.

The ``only if'' portion of Proposition \ref{prop:activation_constraint} thus 
directly suggests constraints to impose such that the desired local linear 
function $\text{act}_\kappa$ is active on its respective output. In particular, 
among the \textbf{non-active local linear functions at } $x_\textsf{c.e.}$, at 
least one must be altered from each of the selector sets $s_j : j \in \{1, 
\dots, M\}\backslash\{\text{sel}_\kappa\}$. The fact that these alterations 
must be made to local linear functions which are not active at the 
counterexample warrants the description of this procedure as ``global 
alteration''.

Finally, however, we note that altering these un-repaired local linear 
functions -- i.e. those not indexed by $\text{act}_\kappa$ -- may create the 
same issue described in Section 
\ref{ssub:preserving_the_initial_safety_condition}. Thus, for any of these 
global alterations additional safety constraints like 
\eqref{eq:repair_beta_constraint} and \eqref{eq:repair_L_constraint} must be 
imposed on the altered parameters.

\section{Main Algorithm} % (fold)
\label{sec:main_algorithm}

Problem \ref{prob:main_problem} permits the alteration of 
linear-layer parameters in the original TLL controller to perform repair. In 
Section \ref{sec:framework}, we developed \emph{constraints} on these 
parameters to perform
\begin{itemize}
	\item first, local alteration to ensure repair of the defective 
		controller at $x_\textsf{c.e.}$; and

	\item subsequently, global alteration to ensure that the repaired local 
		controller is activated at and around $x_\textsf{c.e.}$.
\end{itemize}
The derivations of both sets of constraints implies that they are merely 
sufficient conditions for their respective purposes, so there is no guarantee 
that any subset of them are jointly feasible. Moreover, as a ``repair'' problem, any repairs conducted must 
involve minimal alteration -- Problem \ref{prob:main_problem}\emph{(v)}.

Thus, the core of our algorithm is to employ a convex solver to find 
the minimally altered TLL parameters that also satisfy the local and global 
constraints we have outlined for successful repair with respect to the other 
aspects of Problem \ref{prob:main_problem}. The fact that the local 
repair constraints are prerequisite to the global activation constraints means 
that we will employ a convex solver on two optimization problems \emph{in 
sequence}: first, to determine the feasibility of local repair and effectuate 
that repair in a minimal way; and then subsequently to determine the 
feasibility of activating said repaired controller as required and effectuating 
that activation in a minimal way.
% We devote the remainder of this section to 
% one sub-section for each of two optimization problems used in our algorithm. 
% The whole algorithm is then summarized in a final subsection.

\subsection{Optimization Problem for Local Alteration (Repair)} % (fold)
\label{sub:optimization_problem_for_local_alteration_}
Local alteration for repair starts by identifying the active controller at the 
counterexample, as denoted by the index $\text{act}_\kappa$ for each output of 
the controller, $\kappa$. The local controller for each output is thus the 
starting point for repair in our algorithm, as described in the prequel. From 
this knowledge, an explicit constraint sufficient to repair the local 
controller at $x_\textsf{c.e.}$ is specified directly by the dynamics: see 
\eqref{eq:repair_constraint}.

Our formulation of a safety constraint for the locally repaired controller 
requires additional input, though. In particular, we need to identify constants 
$\beta_\text{max}$ and $L_\text{max}$ such that the non-local controllers 
satisfy \eqref{eq:original_beta_constraint} and 
\eqref{eq:original_L_constraint}. Then Corollary 
\ref{cor:repaired_tll_reach_set_bound} implies that
\eqref{eq:repair_beta_constraint} and \eqref{eq:repair_L_constraint} are constraints that ensure the 
repaired controller satisfies Problem \ref{prob:main_problem}\emph{(i)}. For this we take the naive approach of setting 
$\beta_\text{max} = \beta(\Xi^{(m)}_{N,M})$, and then solving for the smallest 
$L_\text{max}$ that ensures safety for that particular $\beta_\text{max}$. In 
particular, we set
\begin{equation}
	L_\text{max} \negthinspace = \inf \{ L^\prime > 0 \; | \; \beta_\text{max}  \cdot \sum_{k=0}^T {L^\prime}^k = 
		\negthickspace\negthickspace\negthickspace
		\inf_{\overset{x_s \in X_\textsf{safe}}{\scriptscriptstyle x_u \in X_\textsf{unsafe}}}
		\negthickspace\negthickspace
		\lVert x_s - x_u \rVert
	\}.
\end{equation}
% for this choice of $\beta_\text{max}$. 

Given this information the local repair optimization problem can be formulated 
for a multi-output TLL as:
\begin{align}
	\text{\optprob{Local}}: 
	&\raisebox{-7pt}{
		$
		\overset{
			\displaystyle\text{min}
		}{
			\scriptscriptstyle
			\overbar{w}_{\text{act}_\kappa}^\kappa,
			\overbar{b}_{\text{act}_\kappa}^\kappa
		}
		$
	}
	\negthinspace
	\sum_{\kappa=1}^m
	\lVert \llbracket W^\kappa_\ell \rrbracket_{\text{act}_\kappa} \negthickspace - \overbar{w}_{\text{act}_\kappa}^\kappa \negthinspace \rVert \negthinspace + \negthinspace
	\lVert \llbracket b^\kappa_\ell \rrbracket_{\text{act}_\kappa} \negthickspace - \overbar{b}_{\text{act}_\kappa}^\kappa \negthinspace \rVert \notag \\
	\quad\text{ s.t. }
	& f(x_\textsf{c.e.}) + g(x_\textsf{c.e.})
	\left[
	\begin{smallmatrix}
		\overbar{w}_{\text{act}_1}^1 x_\textsf{c.e.} + \overbar{b}_{\text{act}_1}^1 \\
		{\scriptscriptstyle \vdots} \\
		\overbar{w}_{\text{act}_m}^m x_\textsf{c.e.} + \overbar{b}_{\text{act}_m}^m
	\end{smallmatrix}
	\right]
	\not\in
	X_\textsf{unsafe} \notag \\
	& \forall \kappa = 1, \dots, m \;\; . \;\;
		L(
			\lVert \overbar{w}_{\text{act}_\kappa}^\kappa \rVert,
			\lVert \overbar{b}_{\text{act}_\kappa}^\kappa \rVert
		)
		\leq L_\text{max} \notag \\
	& \forall \kappa = 1, \dots, m \;\; . \;\;
		\beta(
			\lVert \overbar{w}_{\text{act}_\kappa}^\kappa \rVert,
			\lVert \overbar{b}_{\text{act}_\kappa}^\kappa \rVert
		)
		\leq \beta_\text{max} \notag \\
	& \forall \kappa = 1, \dots, m \;\; . \;\;
		L(
			\Xi^{(m)}_{N,M}
		)
		\leq L_\text{max} \notag
\end{align}
Note: the final collection of constraints on $L(\Xi^{(m)}_{N,M})$ is necessary 
to ensure that \eqref{eq:original_L_constraint} is satisfied and Corollary 
\ref{cor:repaired_tll_reach_set_bound} is applicable (equation 
\eqref{eq:original_beta_constraint} is satisfied by definition of 
$\beta_\text{max}$).

% subsection optimization_problem_for_local_alteration_ (end)

\subsection{Optimization Problem for Global Alteration (Activation)} % (fold)
\label{sub:optimization_problem_for_global_alteration_}

If the optimization problem \optprob{Local} is feasible, then the local 
controller at $x_\textsf{c.e.}$ can successfully be repaired, and the global 
activation of said controller can be considered. Since we are starting with a 
local linear function we want to be active at and around $x_\textsf{c.e.}$, we 
can retain the definition of $\text{act}_\kappa$ from the initialization of 
\optprob{Local}. Moreover, since Problem \ref{prob:main_problem} preserves the 
selector matrices of the original TLL controller, we will define the selector 
indices, $\text{sel}_\kappa$, in terms of the activation pattern of the 
\emph{original}, defective local linear controller (although this is not 
required by the repair choices we have made: other choices are possible).

Thus, in order to formulate an optimization problem for global alteration, we 
need to define constraints compatible with Proposition 
\ref{prop:activation_constraint} based on the activation/selector indices 
described above. Part \emph{(i)} of the conditions in Proposition 
\ref{prop:activation_constraint} is unambiguous at this point: it says that the 
desired active local linear function, $\text{act}_\kappa$, must have the 
minimum output from among those functions selected by selector set 
$s_{\text{sel}_\kappa}$. Part \emph{(ii)} of the conditions in Proposition 
\ref{prop:activation_constraint} is ambiguous however: we only need to specify 
\emph{one} local linear function from each of the \emph{other} min groups to be 
``forced'' lower than the desired active local linear function. In the face of 
this ambiguity, we select these functions using indices $\iota_j^\kappa : j \in 
\{1, \dots, M\} \backslash \{\text{act}_\kappa\}$ that are defined as follows:
\begin{equation}
	\iota_j^\kappa \triangleq \arg \min_{ i \in s^\kappa_j }
	\llbracket
		W^\kappa_\ell x_\textsf{c.e.} + b^\kappa_\ell
	\rrbracket_{i}.
\end{equation}
That is we form our global alteration constraint out of the non-active 
controllers which are have the \emph{lowest outputs among their respective min 
groups}. We reason that these local linear functions will in some sense require 
the least alteration in order to satisfy Part \emph{(ii)} of Proposition 
\ref{prop:activation_constraint}, which requires their outputs to be less than 
the local linear function that we have just repaired.

Thus, we can formulate the global alteration optimization problem as follows:
\begin{align}
	&\text{\optprob{Global}}: 
	\raisebox{-7pt}{
		$
		\overset{
			\displaystyle\text{min}
		}{
			\scriptscriptstyle
			\overbar{W}_{\ell}^\kappa,
			\overbar{b}_{\ell}^\kappa
		}
		$
	}
	\negthinspace
	\sum_{\kappa=1}^m
	\lVert W^\kappa_\ell - \overbar{W}_\ell \rVert
	\negthinspace + \negthinspace
	\lVert b^\kappa_\ell - \overbar{b}_\ell \rVert \notag \\
	\text{ s.t. }
	&\quad\forall \kappa = \{1, \dots, m\} \hphantom{~ \forall i \in s_{\text{sel}_\kappa}} \;.\; 
		\llbracket W^\kappa_\ell \rrbracket_{\text{act}_\kappa, \cdot}
		= 
		\overbar{w}_{\text{act}_\kappa}^\kappa \notag \\
	& \quad\forall \kappa = \{1, \dots, m\} \hphantom{~ \forall i \in s_{\text{sel}_\kappa}} \;.\; 
		\llbracket b^\kappa_\ell \rrbracket_{\text{act}_\kappa, \cdot}
		= 
		\overbar{b}_{\text{act}_\kappa}^\kappa \notag \\
	&\quad\forall \kappa = \{1, \dots, m\} ~ \forall i \in s_{\text{sel}_\kappa} \;.\;
		\overbar{w}_{\text{act}_\kappa}^\kappa x_\textsf{c.e.}
			+ 
		\overbar{b}_{\text{act}_\kappa}^\kappa \notag \\
		&\hspace{150pt}\leq
		\llbracket W^\kappa_\ell x_\textsf{c.e.} + b^\kappa_\ell \rrbracket_{i} \notag \\
	&\quad\forall \kappa = \{1, \dots, m\} \notag \\
		&\hspace{24pt}\forall j \in \{1, \dots, M\}\backslash\{\text{sel}_\kappa\} \;.\;
		\llbracket W^\kappa_\ell x_\textsf{c.e.} + b^\kappa_\ell \rrbracket_{\iota_j^\kappa} \notag \\
		&\hspace{150pt}\leq
		\overbar{w}_{\text{act}_\kappa}^\kappa x_\textsf{c.e.}
			+ 
		\overbar{b}_{\text{act}_\kappa}^\kappa \notag
		\notag
\end{align}
where of course $\overbar{w}_{\text{act}_\kappa}^\kappa$ and $ 
\overbar{b}_{\text{act}_\kappa}^\kappa$ are the repaired local controller 
parameters obtained from the optimal solution of \optprob{Local}. Note that the 
first two sets of equality constraints merely ensure that \optprob{Global} does 
not alter these parameters.

% subsection optimization_problem_for_global_alteration_ (end)

\subsection{Main Algorithm} % (fold)
\label{sub:main_algorithm}

A pseudo-code description of our main algorithm is shown in Algorithm 
\ref{alg:repair_TLL}, as \texttt{repairTLL}. It collects all of the 
initializations from Section \ref{sec:framework}, Subsection 
\ref{sub:optimization_problem_for_local_alteration_} and Subsection 
\ref{sub:optimization_problem_for_global_alteration_}. Only the functions 
\texttt{FindActCntrl} and \texttt{FindActSlctr} encapsulate procedures defined 
in this paper; their implementation is nevertheless adequately described in 
Subsection \ref{ssub:identifying_the_active_controller_at_} and Proposition 
\ref{prop:activation_constraint}, respectively. The correctness of 
\texttt{repairTLL} follows from the results in those sections.

% subsection main_algorithm (end)
\setlength{\textfloatsep}{5pt}
\IncMargin{0.5em}
\begin{algorithm}[!t]

% \SetKwData{b}{b}

% \SetKwData{b}{b}
\SetKwData{false}{False}
\SetKwData{true}{True}

\SetKwData{succArray}{succArray}
\SetKwData{w}{W}
\SetKwData{b}{b}
\SetKwData{wr}{Wr}
\SetKwData{br}{br}
\SetKwData{i}{i}
\SetKwData{minH}{minHRep}
\SetKwData{hyperplanes}{hyperplanes}
\SetKwData{regionhyperplanes}{regionHyperplanes}
\SetKwData{prevlevel}{PrevLevelHeap}
\SetKwData{nextlevel}{NextLevelHeap}
\SetKwData{top}{TopOfPoset}
\SetKwData{currentregion}{currentRegion}
\SetKwData{ver}{verified}

\SetKwFunction{verifyTLLNN}{TraverseHyperplaneRegions}
\SetKwFunction{solvelp}{SolveLP}
\SetKwFunction{FindSuccessors}{FindSuccessors}
\SetKwFunction{minHRep}{FindMinimalHRepresentation}
\SetKwFunction{push}{push}
\SetKwFunction{newheap}{NewHeap}
\SetKwFunction{addtoheap}{AddToHeap}
\SetKwFunction{isempty}{IsEmpty}
\SetKwFunction{removemax}{RemoveMax}
\SetKwFunction{getlinfn}{GetLinFnOnRegion}

\SetKwFunction{CntRegions}{EstimateRegionCount}
\SetKwFunction{GetHyperplanes}{GetHyperplanes}
\SetKwFunction{dim}{Dimensions}
\SetKwFunction{sat}{SATsolver}
\SetKwFunction{maxx}{Maximize}
\SetKwFunction{init}{init}
\SetKwFunction{initbools}{createBooleanVariables}
\SetKwFunction{append}{Append}
\SetKwFunction{satq}{SAT?}
\SetKwFunction{checkfeas}{CheckFeas}
\SetKwFunction{truevars}{TrueVars}
\SetKwFunction{selecthypers}{GetHyperplanes}
\SetKwFunction{getiis}{GetIIS}
\SetKwFunction{cntunique}{CountAllUniqueSubsets}

\SetKwData{gmaxSafe}{gMaxSafe}
\SetKwData{betaMax}{betaMax}
\SetKwData{d}{dSafe}
\SetKwData{Lmax}{Lmax}
\SetKwData{i}{i}
\SetKwData{IOta}{iota}
\SetKwData{sol}{sol}

\SetKwFunction{repairTLL}{repairTLL}
\SetKwFunction{BEta}{beta}
\SetKwFunction{LL}{L}
\SetKwFunction{init}{Initialize}
\SetKwFunction{solve}{Solve}
\SetKwFunction{findAct}{FindActCntrl}
\SetKwFunction{findSel}{FindActSlctr}
\SetKwFunction{feas}{feasible}
\SetKwFunction{val}{optimalValue}
\SetKwFunction{setLin}{setLinLayer}

\SetKw{Break}{break}
\SetKw{not}{not}
\SetKw{In}{in}
\SetKw{ret}{return}
\SetKw{false}{False}

\SetKwInOut{Input}{input}
\SetKwInOut{Output}{output}

%\SetKwRet{\Return}{return}
\Input{
\hspace{3pt}$f,g$ system dynamics \eqref{eq:system_dynamics} \\
\hspace{3pt}$X_\textsf{ws}$ workspace set \\
\hspace{3pt}$\Xi^{(m)}_{N,M}$ TLL controller to repair \\
\hspace{3pt}$T$ safety time horizon \\
\hspace{3pt}$X_\textsf{safe}$ set of safe states under $\Xi^{(m)}_{N,M}$ \\
\hspace{3pt}$x_\textsf{c.e.}$ counterexample state
% \hspace{15pt}of $\mathcal{L}$ (poset base)
}
\Output{
\hspace{3pt}$\overbar{\Xi}^{(m)}_{N,M}$ repaired TLL controller
}
\BlankLine
\SetKwProg{Fn}{function}{}{end}%
\Fn{\repairTLL{$f,g$,$X_\textsf{ws}$,$\Xi^{(m)}_{N,M}$,$T$,$X_\textsf{safe}$,$x_\textsf{c.e.}$}}{

	\gmaxSafe $\leftarrow$ $\sup_{x_0\in X_\textsf{safe}} \lVert g(x_0) \rVert$

	\BEta(w,b) := $\sup_{x_0\in X_\textsf{safe}} \lVert f(x_0) - x_0 \rVert$

		\hspace{10pt} + \gmaxSafe * w * $\text{ext}(X_\textsf{ws})$ + \gmaxSafe * b

	\LL(w,b) := $L_f$ + $L_g$ * w * $\sup_{x_0\in X_\textsf{safe}} \lVert x_0 \rVert$

		\hspace{10pt} + w * \gmaxSafe + $L_g$ * b

	$\Omega_W$ $\leftarrow$ $\max_{
					w \in \cup_{\kappa = 1}^{m} \{ \llbracket W^\kappa_\ell \rrbracket_j | j = 1, \dots, N\}
				}
				\lVert w \rVert$

	$\Omega_b$ $\leftarrow$ $\max_{
				b \in \cup_{\kappa = 1}^{m} \{ \llbracket b^\kappa_\ell \rrbracket_j | j = 1, \dots, N\}
			}
			\lVert b \rVert$

	\betaMax $\leftarrow$
			\BEta{
			$\Omega_W$,
			$\Omega_b$
			}

	\d $\leftarrow$ $\inf_{\overset{x_s \in X_\textsf{safe}}{\scriptscriptstyle x_u \in X_\textsf{unsafe}}}
		\negthickspace
		\lVert x_s - x_u \rVert$

	\Lmax $\leftarrow$ $\inf \big\{ L^\prime | $ \betaMax*$\sum_{k=0}^T {L^\prime}^k = $\d$\negthickspace\big\}$

	\BlankLine

	$\{\text{act}_\kappa\}_{\kappa=1}^m$ $\leftarrow$ \findAct{$\Xi^{(m)}_{N,M}$, $x_\textsf{c.e.}$}

	$\{\text{sel}_\kappa\}_{\kappa=1}^m$ $\leftarrow$ \findSel{$\Xi^{(m)}_{N,M}$, $x_\textsf{c.e.}$}

	\BlankLine

	\init{\optprob{Local},\{$f,g,\Xi^{(m)}_{N,M},x_\textsf{c.e.}$,\LL,\Lmax,
		\BEta,\betaMax,$\{\text{act}_\kappa\}_{\kappa=1}^m, X_\textsf{unsafe}$\}}

	\BlankLine

	\sol $\leftarrow$ \solve{\optprob{Local}}

	\eIf{\not \sol.\feas{}}{
		\ret \false
	}{
		$\{ (\text{\normalfont{w}}^\kappa,\text{\normalfont{b}}^\kappa) \}_{\kappa=1}^m$ $\leftarrow$ \sol.\val{}
	}

	\BlankLine

	\For{$\kappa$ \In $1, \dots, m$}{
		\For{$j$ \In $\{1, \dots, M\}\backslash\{\text{sel}_\kappa\}$}{
			$\iota^\kappa_j$ $\leftarrow$ $\arg \min_{i \in s_j} \lVert \llbracket W^\kappa_\ell x_\textsf{c.e.} + b^\kappa_\ell \rrbracket_{i} \rVert$
		}
	}

	\init{\optprob{Global},\{$f,g,\Xi^{(m)}_{N,M},x_\textsf{c.e.}$,\LL,\Lmax,
		\BEta,\betaMax,$\{\text{act}_\kappa\}_{\kappa=1}^m, \{\iota^\kappa_j\}_{\kappa,j}, \{(\text{\normalfont{w}}^\kappa,\text{\normalfont{b}}^\kappa)\}$\}}

	\sol $\leftarrow$ \solve{\optprob{Global}}

	\eIf{\not \sol.\feas{}}{
		\ret \false
	}{
		$\{ (\text{\normalfont{W}}^\kappa,\text{\normalfont{B}}^\kappa) \}_{\kappa=1}^m$ $\leftarrow$ \sol.\val{}
	}

	\BlankLine

	\ret $\Xi^{(m)}_{N,M}$.\setLin{$\{ (\text{\normalfont{W}}^\kappa,\text{\normalfont{B}}^\kappa) \}_{\kappa=1}^m$}

}
\caption{\texttt{repairTLL}.}
\label{alg:repair_TLL}
\end{algorithm}
\DecMargin{0.5em}

% section main_algorithm (end)

% \input{lipschitz}

% \input{reachability}

% \input{repair}

% !TEX root = ./main.tex

\section{Numerical Examples}

\vspace{\extraspace} We illustrate the results in this paper on a four-wheel 
car described by the following model:
\begin{equation} \label{eq:29}
    x(t+1) = 
    \begin{bmatrix}
    x_1(t) + V \cos(x_3(t)) \cdot t_s\\
    x_2(t) + V \sin(x_3(t)) \cdot t_s\\
    x_3(t)\\
    \end{bmatrix}+
    \begin{bmatrix}
    0\\
    0\\
    t_s\\
    \end{bmatrix}v(t)
\end{equation}
where the state $x(t) = [p_x(t) \ p_y(t) \  \Psi(t)]^T$ for vehicle position $(p_x \ p_y)$ and yaw angle $\Psi$ , the and control input $v$ 
is the vehicle yaw rate. The parameters are the translational speed of the vehicle, $V$ (meters/sec); 
and the sampling period, $t_s$ (sec). For the purposes of our experiments, we consider 
a compact workspace $X_\textsf{ws} =  [-3,3] \times 
[-4, 4] \times [-\pi, \pi]$; a safe set of states $X_\textsf{safe} =  [-0.25,0.25] \times [-0.75,-0.25] \times 
[-\frac{\pi}{8}, \frac{\pi}{8}]$, which was verified using NNV \cite{TranNNVNeuralNetwork2020} over $100$ iterations; and an unsafe region $X_\textsf{unsafe}$ 
specified by $[0\ 1\ 0]\cdot x > 3$. Furthermore, we consider model parameters: $V=0.3$ m/s
and $t_s=0.01$ seconds.

All experiments were executed on an Intel Core i5 2.5-GHz processor with 8 GB of memory.
We collected $1850$ data points of state-action pairs from a PI Controller used to steer the car over $X_\textsf{ws}$ while avoiding $X_\textsf{unsafe}$. Then, to exhibit a NN controller with a counterexample, a TLL NN with $N=50$ and $M=10$ was trained from a  corrupted version of this data-set: we manually changed the control on $25$ data points close to $X_\textsf{unsafe}$ so that the car would steer into it. We simulated the trajectories of the car using this TLL NN controller for different $x_0$ and identified $x_\textsf{c.e.} =[0 \ 2.999 \ 0.2]$ as a valid counterexample for safety after two time steps. Finally, to repair this \emph{faulty} NN, we found all the required bounds for both system dynamics and NN parameters and a horizon of $T=7$. We found the required safety constraints $\beta_\text{max} = 0.0865$ and $L_\text{max}=1.4243$. Then, from $x_\textsf{c.e.}$ we obtained the controller $K=[K_w \ K_b]$ where $K_w=[-0.1442, \ -0.5424, \ -0.425]$ and $K_b=[2.223]$.

Next, we ran our algorithm to repair the counterexample using CVX (convex solver). The result of the first optimization problem, \optprob{Local}, was the linear controller: $\bar K_w = [-0.0027 \ -0.0487 \ -0.0105]$ and $\bar K_b =[-9.7845]$; this optimization required a total execution time of $1.89$ sec. The result of the second optimization problem, \optprob{Global} successfully activated the repaired controller, and had an optimal cost of $8.97$; this optimization required a total execution time of $6.53$ sec. We also compare the original TLL Norms $||W||=6.54$ and $||b||=5.6876$ with the repaired: $||\overline W||=11.029$ and $||\overline b||=5.687$.

Finally, we simulated the motion of the car using the repaired TLL NN controller for $50$ steps. Shown in Fig. 1 are the state trajectories of both original \emph{faulty} TLL controller and repaired TLL Controller starting from the $x_\textsf{c.e.}$
In the latter the TLL controller met the safety specifications.

\begin{figure}[tbh!]
\centering
\includegraphics[width=0.99\columnwidth]{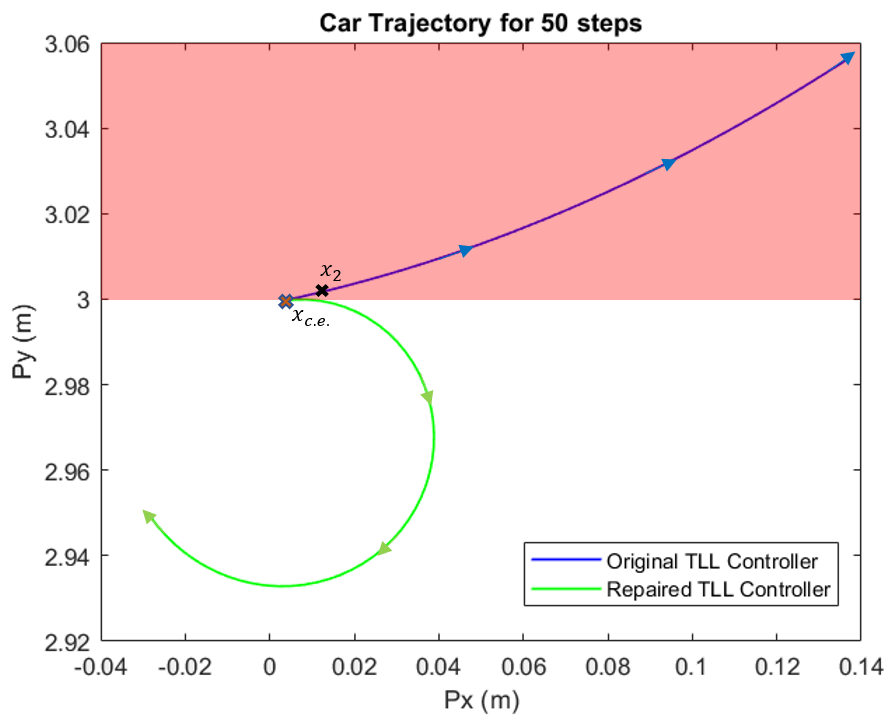}
\vspace{-8pt}
\caption{System starting from $x_\textsf{c.e.}$ goes directly into $X_\textsf{unsafe}$ and Repaired $x_\textsf{c.e.}$ produces a safe trajectory. Red area is $X_\textsf{unsafe}$, Red Cross is $x_\textsf{c.e.}$ and Black Cross shows state after 2 steps.}
\label{fig:6}
\end{figure}
\begin{figure}[tbh!]
\centering
\includegraphics[width=0.99\columnwidth]{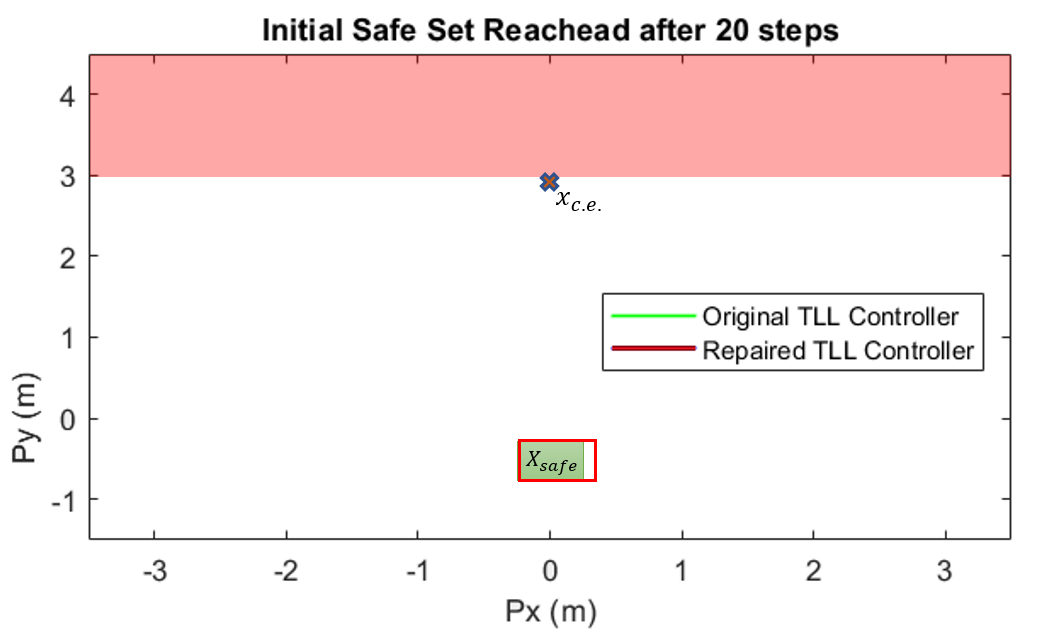}
\vspace{-8pt}
\caption{Initial Safe set before and after repair for 20 steps. Red area is $X_\textsf{unsafe}$; Red Cross is $x_\textsf{c.e.}$; $X_\textsf{ws} =  [-3,3] \times [-4, 4]$}
\label{fig:6}
\end{figure}

% \input{conclusions}

% \input{appendix}

%%%%%%%%%%%%%%%%%%%%%%%%%%%%%%%%%%%%%%%%%%%%%%%%%%%%%%%%%%%%%%%%%%%%%%%%%%%%%%%%

\bibliographystyle{ieeetr}
\bibliography{bibliography}

% \section{Acknowledgements}

% \noindent This research was funded by 

\clearpage

\newpage

% !TEX root = ./main.tex

%%%%%%%%%%%%%%%%%%%%%%%%%%%%%%%%%%%%%%%%%%%%%%%%%%%%%%%%%%%%%%%%%%%%%%%%%%%%%%%%
\section{Appendix}
\label{sec:appendix}

\subsection{Proof of Corollary \ref{cor:active_local_linear_fn}} % (fold)
\label{sub:proof_of_corollary_ref}

\begin{proof}
	It is straightforward to see that every point $x$ in the domain of 
	$\nn\subarg{\Xi^{(m)}_{N,M}}$ belongs to the closure of some open set 
	$\mathfrak{D}$, on which $\nn\subarg{\Xi^{(m)}_{N,M}}$ is affine (i.e. 
	equal to one of its local linear functions). For if this weren't the case, 
	then there would be an open subset of the domain of 
	$\nn\subarg{\Xi^{(m)}_{N,M}}$, where it \emph{wasn't} affine, thus 
	contradicting the CPWA property of a ReLU NN.

	Thus, let $\mathfrak{D}_{x_\textsf{c.e.}}$ be such an open set that 
	includes $x_\textsf{c.e.}$ in its closure, and let $\ell : \mathbb{R}^n 
	\rightarrow \mathbb{R}^m$ be the local linear function of 
	$\nn\subarg{\Xi^{(m)}_{N,M}}$ on $\mathfrak{D}_{x_\textsf{c.e.}}$. We can 
	further assume that $\mathfrak{D}_{x_\textsf{c.e.}}$ is connected without 
	loss of generality, so set $R_a = \overbar{\mathfrak{D}}_{x_\textsf{c.e.}}$.

	By Proposition \ref{prop:local_lin_fns_params}, there exists indices 
	$\{\text{act}_\kappa\}_{\kappa=1}^m$ such that
	\begin{equation}
		\llbracket \ell \rrbracket_{\text{act}_\kappa}
		=
		x \mapsto \llbracket W^\kappa_\ell x + b^\kappa_\ell \rrbracket_{\text{act}_\kappa}.
	\end{equation}
	But by the definition of $\ell$ and the above, we also have that
	\begin{equation}
		\forall x \in \mathfrak{D}_\textsf{c.e.} \;.\;
		\llbracket W^\kappa_\ell x + b^\kappa_\ell \rrbracket_{\text{act}_\kappa}
		= 
		\llbracket \nn\subarg{\Xi^{(m)}_{N,M}}(x) \rrbracket_{\text{act}_\kappa}.
	\end{equation}
	Thus, the conclusion of the corollary holds on 
	$\mathfrak{D}_{x_\textsf{c.e.}}$; it holds on $R_a = 
	\overbar{D}_{x_\textsf{c.e.}}$ by continuity of 
	$\nn\subarg{\Xi^{(m)}_{N,M}}$.
\end{proof}

\subsection{Proof of Proposition \ref{prop:reach_set_bound}} % (fold)
\label{sub:proof_of_proposition_prop:reach_set_bound}

\begin{lemma}\label{lem:lipschitz_prod_lemma}
	Let $F : x \mapsto g(x)\cdot u(x)$ for Lipschitz continuous functions $g : 
	\mathbb{R}^n \rightarrow \mathbb{R}^(n\cdot m)$ (with output an $(n \times 
	m)$ real-valued matrix) and $u : \mathbb{R}^n \rightarrow \mathbb{R}^m$ 
	with Lipschitz constants $L_g$ and $L_u$, respectively.

	Then on compact subset $X \subset \mathbb{R}^n$, $F$ is Lipschitz 
	continuous with Lipschitz constant $L_F = L_g \cdot \sup_{x \in X} \lVert 
	u(x) \rVert + L_u \cdot \sup_{x \in X} \lVert g(x) \rVert$.
\end{lemma}
\begin{proof}
	This follows by straightforward manipulations as follows. Let $x, x^\prime 
	\in X$ and note that:
	\begin{align}
		&\lVert g(x)u(x) - g(x^\prime)u(x^\prime) \rVert \notag \\
		&= \lVert g(x)u(x) + (-g(x^\prime) + g(x^\prime))u(x) - g(x^\prime)u(x^\prime) \rVert \notag \\
		&= \lVert (g(x) - g(x^\prime))u(x) + g(x^\prime)(u(x) - u(x^\prime)) \rVert \notag \\
		&\leq \lVert g(x) - g(x^\prime) \rVert \cdot \lVert u(x) \rVert 
			+ \lVert u(x) - u(x^\prime)\rVert \cdot \lVert g(x^\prime) \rVert \notag \\
		&\leq \big( L_g \cdot \lVert u(x) \rVert + L_u \cdot \lVert g(x^\prime) \rVert \big) \cdot \lVert x - x^\prime \rVert \notag \\
		&\leq \big( L_g \cdot \sup_{x\in X} \lVert u(x) \rVert + L_u \cdot \sup_{x^\prime \in X}\lVert g(x^\prime) \rVert \big) \cdot \lVert x - x^\prime \rVert. \notag
	\end{align}
\end{proof}

\begin{proof}{(Proposition \ref{prop:reach_set_bound})}
	We will expand and bound the quantity on the left-hand side of the 
	conclusion, \eqref{eq:main_bound_prop_conclusion}.
	\begin{align}
		&\lVert \zeta^{x_0}_{T}(\Psi) - x_0 \rVert \notag \\
		&\;\;=\lVert \zeta^{x_0}_{T}(\Psi) - \zeta^{x_0}_{T-1}(\Psi) + \zeta^{x_0}_{T-1}(\Psi) - x_0 \rVert \notag \\
		&\;\;\leq \lVert \zeta^{x_0}_{T}(\Psi) - \zeta^{x_0}_{T-1}(\Psi) \rVert + \lVert \zeta^{x_0}_{T-1}(\Psi) - x_0 \rVert 
		\label{eq:base_case_bound_induction}
		%\notag \\
		% &\;\;\leq \lVert f(\zeta^{x_0}_{T-1}(\Psi)) - f(\zeta^{x_0}_{T-2}(\Psi)) \rVert  \notag \\
		% &+ \Big\lVert g(\zeta^{x_0}_{T-1}(\Psi))\cdot \left[ w(\zeta^{x_0}_{T-1}(\Psi)) \cdot \zeta^{x_0}_{T-1}(\Psi) + b(\zeta^{x_0}_{T-1}(\Psi))\right] \notag \\
		% &- g(\zeta^{x_0}_{T-2}(\Psi))\cdot \left[ w(\zeta^{x_0}_{T-2}(\Psi)) \cdot \zeta^{x_0}_{T-2}(\Psi) + b(\zeta^{x_0}_{T-2}(\Psi))\right] \Big\rVert \notag \\
		% &\qquad\qquad + \lVert \zeta^{x_0}_{T-1}(\Psi) - x_0 \rVert
		% \label{eq:main_bound_inequality}
	\end{align}
	We then bound the first term as follows:
	\begin{align}
		&\lVert \zeta^{x_0}_{T}(\Psi) - \zeta^{x_0}_{T-1}(\Psi) \rVert \notag \\
		&\;\;\leq \lVert f(\zeta^{x_0}_{T-1}(\Psi)) - f(\zeta^{x_0}_{T-2}(\Psi)) \rVert  \notag \\
		&+ \Big\lVert g(\zeta^{x_0}_{T-1}(\Psi))\cdot \left[ w(\zeta^{x_0}_{T-1}(\Psi)) \cdot \zeta^{x_0}_{T-1}(\Psi) + b(\zeta^{x_0}_{T-1}(\Psi))\right] \notag \\
		&- g(\zeta^{x_0}_{T-2}(\Psi))\cdot \left[ w(\zeta^{x_0}_{T-2}(\Psi)) \cdot \zeta^{x_0}_{T-2}(\Psi) + b(\zeta^{x_0}_{T-2}(\Psi))\right] \Big\rVert
		\label{eq:main_bound_inequality}
	\end{align}
	where the functions $w : \mathbb{R}^n \rightarrow \mathbb{R}^n$ and $b : 
	\mathbb{R}^n \rightarrow \mathbb{R}$ return a (unique) choice of the linear 
	(weights) and affine (bias) of the local linear function of $\Psi$ that is 
	active at their argument.

	Now, we collect the $w(\cdot)$ and $b(\cdot)$ terms in right-hand side of 
	\eqref{eq:main_bound_inequality}. That is:
	\begin{align}
		&\lVert \zeta^{x_0}_{T}(\Psi) - \zeta^{x_0}_{T-1}(\Psi) \rVert \notag \\
		&\;\;\leq \lVert f(\zeta^{x_0}_{T-1}(\Psi)) - f(\zeta^{x_0}_{T-2}(\Psi)) \rVert  \notag \\
		&+ \Big\lVert g(\zeta^{x_0}_{T-1}(\Psi)) w(\zeta^{x_0}_{T-1}(\Psi)) \zeta^{x_0}_{T-1}(\Psi) \notag \\ 
		&\qquad\qquad\qquad - g(\zeta^{x_0}_{T-2}(\Psi) w(\zeta^{x_0}_{T-2}(\Psi)) \zeta^{x_0}_{T-2}(\Psi) \big\rVert \notag \\
		&+ \Big\lVert g(\zeta^{x_0}_{T-1}(\Psi)) b(\zeta^{x_0}_{T-1}(\Psi)) - g(\zeta^{x_0}_{T-2}(\Psi) b(\zeta^{x_0}_{T-2}(\Psi)) \big\rVert \notag %\\
		% &\qquad\qquad\qquad + \lVert \zeta^{x_0}_{T-1}(\Psi) - x_0 \rVert \notag
	\end{align}
	The first term in the above can be directly bounded using the Lipschitz 
	constant of $f$. Also, since there are only finitely many local linear 
	function of $\Psi$, $b(\cdot)$ takes one of finitely many values across the 
	entire state space, and we may bound the associated term using this 
	observation. Finally, we can Lemma \ref{lem:lipschitz_prod_lemma} to the 
	second term, noting that the linear function defined by $w(\cdot)$ has 
	Lipschitz constant $\lVert w(\cdot) \rVert$ and there are only finitely 
	many possible values for this quantity (one for each local linear 
	function). This yields the following bound:
	\begin{align}
		&\lVert \zeta^{x_0}_{T}(\Psi) - \zeta^{x_0}_{T-1}(\Psi) \rVert \notag \\
		&\;\;\leq L_f \cdot \lVert \zeta^{x_0}_{T-1}(\Psi) - \zeta^{x_0}_{T-2}(\Psi) \rVert  \notag \\
		&+ \big(
			L_g \cdot \sup_{x\in X_\textsf{ws}}\lVert w(x) \cdot x \rVert + \max_k \lVert w_k \rVert \sup_{x\in X_\textsf{ws}} \lVert g(x) \rVert
		\big) \notag \\
		&\qquad\qquad \cdot \lVert \zeta^{x_0}_{T-1}(\Psi) - \zeta^{x_0}_{T-2}(\Psi) \rVert \notag \\
		&+ \max_k \lVert b_k \rVert \cdot L_g \cdot \lVert \zeta^{x_0}_{T-1}(\Psi) - \zeta^{x_0}_{T-2}(\Psi) \rVert
		\notag% \\
		% &\qquad\qquad\qquad + \lVert \zeta^{x_0}_{T-1}(\Psi) - x_0 \rVert \notag
	\end{align}
	If we simplify, then we see that we have
	\begin{multline} %\label{eq:base_case_bound_induction}
		\lVert \zeta^{x_0}_{T}(\Psi) - \zeta^{x_0}_{T-1}(\Psi) \rVert \\
		\leq L_\text{max}(\Psi) \cdot 
		\lVert \zeta^{x_0}_{T-1}(\Psi) - \zeta^{x_0}_{T-2}(\Psi) \rVert 
		% \\
		% + 
		% \lVert \zeta^{x_0}_{T-1}(\Psi) - x_0 \rVert
		\label{eq:induction_step_eq}
	\end{multline}
	with $L_\text{max}(\Psi)$ as defined in the statement of the Proposition.

	Now, we expand the final term of \eqref{eq:base_case_bound_induction} as 
	\begin{multline}
		\lVert \zeta^{x_0}_{T-1}(\Psi) - x_0 \rVert \\
		\leq \lVert 
		\zeta^{x_0}_{T-1}(\Psi) - \zeta^{x_0}_{T-2}(\Psi) \rVert + \lVert 
		\zeta^{x_0}_{T-2}(\Psi) - x_0 \rVert
	\end{multline}
	so that \eqref{eq:base_case_bound_induction} can be rewritten as:
	\begin{multline}
		\lVert \zeta^{x_0}_{T}(\Psi) - x_0 \rVert \leq (L_\text{max}(\Psi) + 1)\cdot 
		\lVert \zeta^{x_0}_{T-1}(\Psi) - \zeta^{x_0}_{T-2}(\Psi) \rVert \\
		+ 
		\lVert \zeta^{x_0}_{T-2}(\Psi) - x_0 \rVert.
		\label{eq:real_base_case}
	\end{multline}

	But now we can proceed inductively, applying the bound 
	\eqref{eq:induction_step_eq} mutatis mutandis to the expression $\lVert 
	\zeta^{x_0}_{T-1}(\Psi) - \zeta^{x_0}_{T-2}(\Psi) \rVert$ in 
	\eqref{eq:real_base_case}. This induction can proceed until the factor to 
	be expanded using \eqref{eq:induction_step_eq} has the form $\lVert 
	\zeta^{x_0}_{T-(T-1)}(\Psi) - \zeta^{x_0}_{T-(T)}(\Psi) \rVert$, which will 
	yield the bound:
	\begin{equation}
		\lVert \zeta^{x_0}_{T}(\Psi) - x_0 \rVert \leq
		\lVert \zeta^{x_0}_1(\Psi) - x_0 \rVert \cdot \sum_{k=0}^T {L_\text{max}(\Psi)}^k.
	\end{equation}

	Thus it remains to bound the quantity $\lVert \zeta^{x_0}_1(\Psi) - x_0 
	\rVert$. We proceed to do this in a relatively straightforward way:
	\begin{align}
		&\lVert \zeta^{x_0}_1(\Psi) - x_0 \rVert \notag \\
		&= \lVert f(x_0) + g(x_0)\left[ w(x_0) x_0 + b(x_0) \right] - x_0 \rVert \notag \\
		&\leq \lVert f(x_0) - x_0 \rVert + \lVert g(x_0) \rVert \cdot \lVert w(x_0) \rVert \cdot \lVert x_0 \rVert \notag \\
		&\qquad\qquad + \lVert g(x_0) \rVert \cdot \lVert b(x_0) \rVert.
	\end{align}

	Finally, since we're interested in bounding the original quantity, $\lVert 
	\zeta^{x_0}_T(\Psi) - x_0 \rVert $, over all $x_0 \in X_\textsf{safe}$, we 
	can upper-bound the above by taking a supremum over all $x_0 \in 
	X_\textsf{safe}$. Thus,
	\begin{multline}
	 	\sup_{x \in X_\textsf{safe}} \lVert \zeta^{x_0}_{T}(\Psi) - x_0 \rVert \\
	 	\leq
		\sup_{x \in X_\textsf{safe}} \negthinspace \lVert \zeta^{x_0}_1(\Psi) - x_0 \rVert \cdot \sum_{k=0}^T {L_\text{max}(\Psi)}^k
	\end{multline} 
	where the $\sup$ on the right-hand side does not interact with the 
	summation, since $L_\text{max}(\Psi)$ is constant with respect to $x_0$. 
	The final conclusion is obtained by observing that that
	\begin{equation}
		\sup_{x \in X_\textsf{safe}} \negthinspace \lVert \zeta^{x_0}_1(\Psi) - x_0 \rVert \leq \beta_\text{max}(\Psi)
	\end{equation}
	with $\beta_\text{max}(\Psi)$ as defined in the statement of the 
	proposition.
\end{proof}

% subsection proof_of_proposition_prop:reach_set_bound (end)

\subsection{Proof of Proposition \ref{prop:reach_set_bound_safe}} % (fold)
\label{sub:proof_of_proposition_prop:reach_set_bound_safe}

\begin{proof}
	This is more or less a straightforward application of Proposition 
	\ref{prop:reach_set_bound}.

	Indeed, by Proposition \ref{prop:reach_set_bound} and the assumption of 
	this proposition, we conclude that
	\begin{align}
		\lVert \zeta^{x_0}_T(\Psi) - x_0 \lVert
			&\leq 
			\beta_\text{max}(\Psi) \cdot \sum_{k=0}^T {L_\text{max}(\Psi)}^k \notag \\
		&\leq
			\beta_\text{max} \cdot \sum_{k=0}^T {L_\text{max}}^k. \notag
	\end{align}
	Hence, $\delta x = \zeta^{x_0}_T(\Psi) - x_0$ triggers the implication in 
	\eqref{eq:safe_reach_constants}, and we conclude that
	\begin{equation}
		\forall x_0 \in X_\textsf{safe} \;.\; x_0 + \delta x = \zeta^{x_0}_T(\Psi) \not\in X_\textsf{unsafe}
	\end{equation}
	as required.
\end{proof}

% subsection proof_of_proposition_prop:reach_set_bound_safe (end)

\subsection{Proof of Corollary \ref{cor:repaired_tll_reach_set_bound}} % (fold)
\label{sub:proof_of_corollary_cor:repaired_tll_reach_set_bound}

\begin{proof}
	Corollary \ref{cor:repaired_tll_reach_set_bound} is simply a 
	particularization of Proposition \ref{prop:reach_set_bound_safe} to the 
	repaired TLL network, $\overbar{\Xi}^{(m)}_{N,M}$. It is only necessary to 
	note that we have separate conditions to ensure that conclusion of 
	Proposition \ref{prop:reach_set_bound_safe} applied both to the 
	\emph{original} TLL network (i.e. \eqref{eq:original_beta_constraint} and 
	\eqref{eq:original_L_constraint}), as well as the repaired TLL parameters 
	(i.e. \eqref{eq:repair_beta_constraint} and \eqref{eq:repair_L_constraint}).
\end{proof}

% subsection proof_of_corollary_cor:repaired_tll_reach_set_bound (end)

\subsection{Proof of Proposition \ref{prop:activation_constraint}} % (fold)
\label{sub:proof_of_proposition_prop:activation_constraint}

\begin{proof}
	The ``if'' portion of this proof is suggested by the computations in 
	Section \ref{ssub:identifying_the_active_controller_at_}, so we focus on 
	the ``only if'' portion.

	Thus, let $\{\text{act}_\kappa\}_{\kappa=1}^m \in \{1, \dots, N\}^m$ be a 
	set of indices, and assume that there exists an index 
	$\{\text{sel}_\kappa\}_{\kappa=1}^m \in \{1, \dots, M\}^m$ for which the 
	``only if'' assumptions of the proposition are satisfied. We will show that 
	the local linear function with indices 
	$\{\text{act}_\kappa\}_{\kappa=1}^m$ is in fact active on $R_a$.

	This will follow more or less directly by simply carrying out the 
	computations of the TLL NN on $R_a$. In particular, by condition 
	\emph{(i)}, we have that $\nn\subarg{\Theta_{\min_N} \negthickspace \circ 
	\Theta_{S^\kappa_{\text{sel}_\kappa}} \negthickspace \negthinspace \circ 
	\Theta_\ell^\kappa} = x \mapsto \llbracket W^\kappa_\ell x + b^\kappa_\ell 
	\rrbracket_{\text{act}_\kappa}$ for all $x \in R_a$, $\kappa=1, \dots, m$. 
	Then, by condition \emph{(ii)} we have that for all $j \in \{1, \dots, M\} 
	\backslash \{\text{sel}_\kappa\}$ and $x \in R_a$
	\begin{equation}\label{eq:ordering_of_min_groups}
		\nn\subarg{\Theta_{\min_N} \negthickspace \circ 
			\Theta_{S^\kappa_{\text{sel}_\kappa}} \negthickspace \negthinspace \circ 
			\Theta_\ell^\kappa
		}(x)
		\leq
		\nn\subarg{\Theta_{\min_N} \negthickspace \circ 
			\Theta_{S^\kappa_{\text{sel}_\kappa}} \negthickspace \negthinspace \circ 
			\Theta_\ell^\kappa
		}(x).
	\end{equation}
	The conclusion thus follows immediately from 
	\eqref{eq:ordering_of_min_groups} and the fact that the $\min$ groups for 
	the input to the final layer of the output's TLL, $\Theta_{\max_M}$.
\end{proof}

% subsection proof_of_proposition_prop:activation_constraint (end)

% subsection proof_of_corollary_ref (end)

\end{document}